\documentclass[11pt]{article}
\usepackage{fullpage}
\usepackage{hyperref}       
\usepackage{url}    
\usepackage{booktabs}       
\usepackage{nicefrac}       
\usepackage{microtype}  
\usepackage{times}   
\usepackage{xcolor,colortbl} 
\usepackage{multirow}
\usepackage{tablefootnote}
\usepackage{subcaption}
\usepackage{preamble}
\usepackage{mathtools}

\usepackage{natbib}

\usepackage{authblk}

\begin{document}

\title{Offline and Online KL-Regularized RLHF under Differential Privacy}

\author[1]{Yulian Wu} 
\author[2]{Rushil Thareja}
\author[2,3]{ Praneeth Vepakomma}
\author[1]{Francesco Orabona}
\affil[1]{King Abdullah University of Science and Technology (KAUST)}
\affil[2]{Mohamed bin Zayed University of Artificial Intelligence (MBZUAI)}
\affil[3]{Massachusetts Institute of Technology (MIT)}

\date{}

\maketitle

\begin{abstract}
  In this paper, we study the offline and online settings of reinforcement learning from human feedback (RLHF) with KL-regularization---a widely used objective function in large language model alignment---under the $\epsilon$ local differential privacy ($\epsilon$-LDP) model on the label of the human preference. In the offline setting, we design an algorithm based on the principle of pessimism and derive a new suboptimality gap of $\tilde{O}(1/[(e^\epsilon-1)^2 n])$ on the KL-regularized objective under single-policy concentrability. We also prove its optimality by providing a matching lower bound where $n$ is the sample size.
  In the online setting, we are the first one to theoretically investigate the problem of KL-regularized RLHF with LDP. We design an optimism-based algorithm and derive a logarithmic regret bound of $O(d_{\mathcal{F}}\log (N_{\mathcal{F}}\cdot T) /(e^\epsilon-1)^2 )$, where $T$ is the total time step, $N_{\mathcal{F}}$ is cardinality of the reward function space $\mathcal{F}$ and $d_{\mathcal{F}}$ is a variant of eluder dimension for RLHF. As a by-product of our analysis, our results also imply the first analysis for online KL-regularized RLHF without privacy. We implement our algorithm in the offline setting to verify our theoretical results and release our open source code at: \url{https://github.com/rushil-thareja/PPKL-RLHF-Official}.

  

\end{abstract}
\newpage
\tableofcontents
\newpage

\section{INTRODUCTION}

The alignment of Large Language Models (LLMs) with human preferences, often achieved through Reinforcement Learning from Human Feedback (RLHF), has become a central area of research. A key technique in this process is the Kullback-Leibler (KL) regularization, which is widely used to prevent the model from deviating too far from its original behavior and to avoid overfitting~\citep{zhao2024sharp,aminian2025theoretical,zhao2025sharpanalysisofflinepolicy,xiong2023iterative}. Mathematically, this objective function encourages the maximization of a reward model while forcing the learned policy $\pi$ to stay close to a base policy $\pi_{\mathrm{ref}}$ for a given state $s$ (prompt) and action $a$ (response):
\begin{equation}\label{eq:ReguObj}
    J(\pi)
    :=\mathbb{E}_{(s, a) \sim d_0 \times \pi}\left[r^*(s, a)-\beta^{-1} \log \frac{\pi(a \mid s)}{\pi_{\mathrm{ref}}(a \mid s)}\right],
\end{equation}
where $r^*(s,a)$ represents the ground truth reward and $\beta>0$ is the inverse temperature parameter. The performance of algorithms is measured by the suboptimality gap in the offline setting, defined as
\begin{equation}\label{eq:OfflineSubOpt}
    \text{SubOpt}(\pi)
    :=J(\pi^*)-J(\pi),
\end{equation}
where $\pi^*$ is the optimal policy $\pi^*:=\arg\max_\pi J(\pi)$. In the online setting, performance is measured by regret:
\begin{equation}\label{eq:OnlineReg}
    \text{Reg}(\pi_{1:T})
    :=\sum_{t=1}^T (J(\pi^*)-J(\pi_t))~.
\end{equation}

While RLHF is effective, significant privacy concerns arise because the preference data used for alignment may contain personal or sensitive information~\citep{zhang2025carefulfinetuningopensourcellms,su2025largelanguagemodelsreally}. The standard framework for quantifying and mitigating privacy leakage is Differential Privacy (DP)~\citep{dwork2014algorithmic}. By introducing calibrated randomness, DP ensures that the output of an algorithm is not overly sensitive to any single individual's data, thereby protecting their privacy. In the context of learning from human feedback, a key challenge is to preserve the privacy of the potentially sensitive preference labels provided by users. This has motivated recent work on applying DP specifically to preference-based learning, often referred to as label differential privacy (label DP)~\citep{ghazi2021deep}.  Label differential privacy in KL-regularized RLHF for the offline setting is studied in \citet{zhang2025klregularizationdifferentiallyprivatebandits} under a central differential privacy model in which the learner can access the raw information of human labels. However, in some applications, individual labelers may be unwilling---or legally unable---to share raw feedback with the learner. These considerations motivate studying a local model for label differential privacy, where each human preference is privatized before disclosure.

Several recent works consider privacy issues on preference labels and study the problem by adopting differential privacy. However, the intersection of these two areas---KL-regularized RLHF and local model label differential privacy---remains unexplored. In particular, it is unknown whether applying label LDP to KL-regularized RLHF can yield strong theoretical guarantees on suboptimality and regret. Motivated by this gap, we are interested in our first question:

\textit{1. In the offline setting, can we achieve an optimal rate for KL-regularized RLHF under the label-LDP setting?}

A primary challenge in offline RLHF is the distribution shift, which occurs when the data distribution used to train the reward model mismatches the response distribution of the optimized policy. This can lead to out-of-distribution errors, reward over-optimization, and degraded performance. While many recent works on theoretical offline RLHF derive rates that depend on notions of data coverage, one effective method to mitigate distribution shift is to use an online version of RLHF. For instance, \citet{zhao2025logarithmic} achieves logarithmic regret for online KL-regularized RL, depending on the eluder dimension.
However, no existing work has studied the privacy problem in online KL-regularized RLHF, which leads us to our second question:

\textit{2. In the online setting, can we provide a logarithmic regret bound for KL-regularized RLHF under a local differential privacy mechanism?}

We answer both of these questions affirmatively and summarize our contributions as follows:
\begin{itemize}
    \item For the problem of private KL-regularized RLHF in the offline setting, we propose the PPKL-RLHF algorithm (Algorithm \ref{algo:offline}), which uses a Random Response (RR) mechanism to achieve label $\epsilon$-LDP. Using these privatized preference labels for a private Maximum Likelihood Estimation (MLE), we obtain a conservative reward estimation via the principle of pessimism, which is then used for policy optimization with Gibbs sampling. We derive a suboptimality gap upper bound  of $\widetilde{O}\left({1}/{[(e^\epsilon-1)^2 n]}\right)$ (Equation \eqref{eq:OfflineSubOpt}), with sample size $n$ and under single policy concentrability. To demonstrate optimality, we also establish a matching lower bound.
    \item For the online setting, we design the POKL-RLHF algorithm (Algorithm \ref{algo:online}), which uses RR to locally privatize human feedback. With the privatized labels and historical data, we design an exploitation agent using private least squares estimation and strategically design exploration via optimism for reward estimation. This exploration strategy yields a logarithmic regret bound for the exploration agent (Equation \ref{eq:OnlineReg}). To the best of our knowledge, we are the first to study the private online KL-regularized RLHF problem.
    \item As a by-product, our analysis provides insights into the non-private online KL-regularized RLHF setting. In particular, we establish the first logarithmic regret bound for online KL-regularized RLHF using a new variant of the eluder dimension. This result outperforms the sublinear regret bound for online RLHF in \citet{xiong2023iterative,xie2024exploratory} and sheds light on future research directions, such as online $f$-regularized RLHF or analyzing online KL-regularized RLHF from a Markov decision process perspective.
    \item Finally, we also run some experiments on a real dataset by implementing our algorithm design for the offline setting.
\end{itemize}


    

\section{RELATED WORK}

Given the large literature on trustworthy LLM alignment, this is necessarily a short review of the most related theory work.
We refer the reader to \citet{liu2023trustworthy} for a more comprehensive survey of this topic. 


\noindent\textbf{Non-Private Offline KL-regularized RLHF} Offline RLHF suffers from a distribution shift problem, since the model is trained on a fixed dataset. Coverage conditions are used to measure the ability of the training-data distribution to cover the test-data distribution. With sample size $n$ in KL-regularized RLHF, \citet{xiong2023iterative} derives a suboptimality gap of\footnote{We use $\widetilde{O}(\cdot), \widetilde{\Omega}(\cdot), \widetilde{\Theta}(\cdot)$ to hide polylog factors.} $\widetilde{O}(1/\sqrt{n})$ under single-policy coverage. \citet{zhao2024sharp} achieves a suboptimality gap of $\widetilde{O}(1/{n})$ but under 
their all-policy concentrability, which is a strong condition that requires the sample distribution to cover all possible distributions. \citet{zhao2025nearly} first establishes the suboptimality gap of $\widetilde{\Theta}(1/n)$ under single-policy coverage.  Building on these, we derive the optimal convergence of $\widetilde{\Theta}(1/[(e^\epsilon-1)^2 n]$ with single-policy concentrability for the private offline KL-regularized RLHF under $\epsilon$-LDP.

\noindent\textbf{Non-private Online KL-regularized RLHF} Online methods are a promising approach to overcome the out-of-distribution problems in offline RLHF. \citet{xiong2023iterative,xiong2024building} show the benefits of the online exploration agent and provides regret of $\widetilde{O}(\sqrt{T})$ for online KL-regularized RLHF with an eluder-type condition. \citet{ye2024online} investigate the online KL-regularized RLHF problem via a Nash equilibrium reformulation. \citet{xie2024exploratory} study online KL-regularized RLHF via adding an exploration term on their loss function based on optimism in the face of uncertainty, and establishes regret of $\widetilde{O}(\sqrt{T})$ under their trajectory-level coverability coefficient. Our result improves has a better regret, but for a different objective function. In fact, taking the privacy parameter $\epsilon \rightarrow +\infty$, our results imply the first logarithmic regret bound of $\widetilde{O}(\log T)$ depending on the eluder dimension.

\noindent\textbf{Locally Private RLHF}
\citet{zhou2025unifiedtheoreticalanalysisprivate,zhou2025square} achieve sub-optimality gap of $\widetilde{O}(1/[(e^\epsilon-1)\sqrt{n}])$ for locally private RLHF on the unregularized suboptimality gap as the performance measure for policy in the offline setting. We adopt a KL-regulized objective function to evaluate progress on the same function the algorithm optimizes, which avoids evaluation–training mismatch. With KL-regularized performance measure, we can improve the sub-optimality gap to $\widetilde{\Theta}(1/(1/[(e^\epsilon-1)^2{n}])$ for the offline setting and achieve $\widetilde{O}(\log T/(e^\epsilon-1)^2)$ with eluder dimension for the online setting, due to the strongly convexity of the KL-regularized objection function.
\citet{chowdhury2024differentially} considers label DP in both local and central models in offline RLHF, but they focus on the estimation error of the parameter, not suboptimality gaps.


\section{PRELIMINARY}
In this section, we introduce the necessary background of KL-regularized RLHF via the contextual bandits view, for both offline and online settings, as well as the basic knowledge of privacy in human feedback. We refer the readers to \cite{li2025provably} for a unified view of RLHF via contextual bandits.

\subsection{Offline and Online KL-regularized RLHF}

KL-regularized RLHF seeks to optimize a target policy $\pi$ by using human preferences to learn a reward function $r(s,a)$, while constraining the policy update to stay close to a reference policy $\pi_{\text{ref}}$. Without loss of generality, we will assume $r(s,a)$ in $[0,B]$ (e.g., via clipping in \citet{huang2025correctingmythosklregularizationdirect} or normalization). This leads to the following objective function:
\begin{equation}\label{eq:ObjFun}
    \max_{\pi} \mathbb{E}_{s \sim d_0,\, a \sim \pi(\cdot \mid x)} [ r(s,a) ] - \frac{1}{\beta} \mathrm{KL}(\pi(\cdot \mid s) \, \| \, \pi_{\text{ref}}(\cdot \mid s)),
\end{equation}
where $\pi_{\text{ref}}$ is often a reference policy (e.g., SFT model).
It is easy to see that the optimal solution of \eqref{eq:ObjFun} is the Gibbs distribution, that is
\begin{equation}
\label{eq:PolicyImprove}
        \pi_r^*(a \mid s)=\frac{1}{Z_r(s)} \pi_{\mathrm{ref}}(a \mid s) \exp (\beta \cdot r(s, a)),
\end{equation}
where $Z_r(s)$ is the normalization constant.

\noindent\textbf{Offline KL-regularized RLHF} 
In the offline case, the learning agent aims to learn a good policy from a pre-collected dataset $\mathcal{D}=\{(s_i, a_i^1, a_i^2, y_i)\}_{i=1}^n$, where $y_i \in \{-1,1\}$ denotes the human's preference between two candidate responses $a_i^1, a_i^2$ generated from the reference policy $\pi_{\mathrm{ref}}$ given a prompt $s_i$ sampled from $d_0$. The binary label $y_i \in \{-1,1\}$ indicates whether $a_i^1 \succ a_i^2$ ($y_i = 1$) or $a_i^2 \succ a_i^1$ ($y_i = -1$), that is, which response is preferred. 

\begin{remark}
We use $y \in \{-1,1\}$ here, which is also adopted in \cite{zhou2025square}, not in $\{0,1\}$ as in most of the RLHF literature, since this will help us simplify the math. The analysis under either convention can be translated back and forth without loss of generality.
\end{remark}




We will need some definitions to quantify the ``concentrability'' of $\pi_\text{ref}$, that is, its ability to generate a diverse set of actions.
\begin{definition}[\citealp{zhao2025sharpanalysisofflinepolicy}]\label{def:Ddivergence}
    Given a class of functions $\mathcal{F} \subset(\mathcal{S} \times \mathcal{A} \to [0,B])$ and some policy $\pi$, let $\mathcal{B}=(\mathcal{S} \to [-B,B])$ be the function class of biases, and define $D_{\mathcal{F}}^2((s, a) ; \pi)$ as
    \[
        \sup _{g, h \in \mathcal{F}} \ \inf_{b \in \mathcal{B}} \ \frac{(g(s, a)-h(s, a)-b(s))^2}{\mathbb{E}_{s^\prime \sim d_0}\operatorname{Var}_{a^{\prime} \sim \pi\left(\cdot \mid s^{\prime}\right)}\left[g\left(s^{\prime}, a^{\prime}\right)-h\left(s^{\prime}, a^{\prime}\right)\right]}~.
    \]
\end{definition}

\begin{definition}[Single-policy Concentrability \citep{zhao2025sharpanalysisofflinepolicy}]\label{def:singleConcen}
    $D_{\pi^*}^2:=\mathbb{E}_{(s, a) \sim d_0 \times \pi^*} D_{\mathcal{F}}^2\left((s, a) ; \pi_{\mathrm{ref}}\right)<\infty$
\end{definition}

\begin{definition}[Density-ratio-based concentrability]\label{def:DensityRatio}
   For policy class $\Pi$ and reference policy $\pi_{\mathrm{ref}}$, the density-ratio-based all-policy concentrability $C^{\Pi}$ is $C^{\Pi}:=\sup _{\pi \in \Pi, s \in \mathcal{S}, a \in \mathcal{A}} \ \pi(a \mid s) / \pi_{\text {ref }}(a \mid s)$. The single-policy counterpart under the optimal policy $\pi^*$ is $C^{\pi^*}:=\sup _{s \in \mathcal{S}, a \in \mathcal{A}} \ \pi^*(a \mid s) / \pi_{\text {ref }}(a \mid s)$.
\end{definition}

\noindent\textbf{Online KL-regularized RLHF}
Online KL-regularized RLHF updates the policy $\pi_t$ over rounds. At each step $t$, a context $s_t$ is drawn, two actions $a_t^1 \sim \pi^1_t$ and $a_t^2 \sim \pi^2_t$ are sampled (possibly asymmetrically), and human feedback $y_t \in \{-1,1\}$ is queried. The second policy $\pi^2_t$ is used to facilitate exploration. 
Based on accumulated feedback $\mathcal{D}_t = \{(s_i, a^1_i, a^2_i, y_i)\}_{i=1}^t$, the reward is re-estimated to get $\hat{r}_t$, and the next policy is updated via \eqref{eq:PolicyImprove}:
\[
\pi^1_{t+1}(a \mid s) 
\propto \pi_{\text{ref}}(a \mid s) \cdot \exp\left( \beta \cdot \hat{r}_t(s,a) \right)~.
\]

\begin{definition}[Uncertainty and pair eluder dimension]\label{def:Uncert}
    For any sequence $\mathcal{D}_{t-1}=\left\{\left(s_i, a_i^1,a_i^2\right)\right\}_{i=1}^{t-1}$, we define $U_{\mathcal{F}}\left(\lambda, s, a; \mathcal{D}_{t};\pi_{t+1}\right)$, the uncertainty of $(s, a)$ with respect to $\mathcal{F}$, as
    \[
    \sup_{r_1, r_2 \in \mathcal{F}}\tfrac{\left|r_1(s, a)-r_2(s, a)-\mathbb{E}_{b\sim \pi_{t+1}}[r_1(s, b)-r_2(s, b)]\right|}{\sqrt{\lambda+\sum_{i=1}^{t}\left(r_1\left(s_i, a_i^1\right)-r_1\left(s_i, a_i^2\right)-[r_2\left(s_i, a_i^1\right)-r_2\left(s_i, a_i^2\right)]\right)^2}}.
    \]
    The pair eluder dimension is given by $d_\mathcal{F}:=\sup_{s_{1:T},a_{1:T}^2} \ \sum_{t=1}^T \min \left(1,\left[U_{\mathcal{F}_t}\left(\lambda, s_t,a_t^2 ; \mathcal{D}_{t};\pi_{t+1}^1\right)\right]^2\right)$. 
    \end{definition} 
\begin{remark}
The eluder dimension definition was first proposed by \cite{russo2013eluder} for multi arm bandits problem to measure the efficacy with which observed data support inference about the values of unobserved actions and then widely used in RL problem~\citep{osband2014model,zhao2025logarithmic,wang2020reinforcement,xie2022role,ye2023corruption,agarwal2023vo,zhong2022gec} and preference-based RL~\citep{wu2023making,chen2022human,ye2024online}. Our definition is a variant of the eluder dimension for the design of the exploration strategy based on the exploitation agent. 
\end{remark}

For both offline and online setting, we adopt the standard Bradley-Terry (BT) model for the preference model and we will assume realizability.
\begin{assumption}[Bradley-Terry Preference Model]\label{Assum:BT}
    Given a context $s$ and two actions $a_1,a_2$, we assume the preference label $y$ is sampled according to the the ground truth reward function $r^*$ difference between the two actions:
    \begin{equation}\label{eq:BTmodel}
        \mathbb{P}[y=1 \mid s, a^1, a^2] = \sigma(r^*(s, a^1) - r^*(s, a^2)),
    \end{equation}
    where $\sigma(x) = (1 + e^{-x})^{-1}$ is the sigmoid function.
\end{assumption}

\begin{assumption}[Realizability of reward function]\label{Assum:Realize}
    We assume that $r^* \in  \mathcal{F} \subset (\mathcal{S}\times\mathcal{A}\rightarrow[0,B])$.
\end{assumption}

To derive uniform theoretical guarantees when 
$|\mathcal{F}|$ is infinite, we approximate it by a finite subset that is sufficiently dense with respect to an appropriate metric. This allows us to apply analysis to the finite subset and then transfer the bound to the entire class via a discretization argument. The complexity of 
$\mathcal{F}$ in this sense is captured by the covering number, which measures how many elements are required to approximate every function in 
$\mathcal{F}$ within a prescribed tolerance. We recall the formal definition below.

\begin{definition}[Net and covering number]\label{def:Net}
Given a function class $\mathcal{F} \subset(\mathcal{S} \times \mathcal{A} \rightarrow [0,B])$ and $\tau\in (0,1)$, a finite set $\mathcal{F}(\tau) \subset \mathcal{F}$ is a $\tau$-net of $\mathcal{F}$ w.r.t. $\|\cdot\|_{\infty}$, if for any $f \in \mathcal{F}$, there exists $f^{\prime} \in \mathcal{F}(\tau)$ such that $\left\|f-f^{\prime}\right\|_{\infty} \leq \tau$. The $\tau$-covering number is the smallest cardinality $\mathcal{N}_{\mathcal{F}}(\tau)$ of such $\mathcal{F}(\tau)$.    
\end{definition}

\subsection{Privacy in Human Feedback}
\label{sec:label_ldp}

Here, we formally introduce the Label Differential Privacy in the local model.
\begin{definition}[$\varepsilon$-Pure Local Label DP~\citep{chowdhury2024differentially}]
    If each label is first privatized by a local randomizer $\mathcal{R}$, which satisfies for any $y, y^{\prime}$ and any subset $S$ in the range of $\mathcal{R}$, it holds that for $\varepsilon>0$,
    \[
    \mathbb{P}[\mathcal{R}(y) \in S] 
    \leq e^{\varepsilon} \cdot \mathbb{P}\left[\mathcal{R}\left(y^{\prime}\right) \in S\right],
    \]
    then, we say that $\mathcal{R}$ is an $\varepsilon$-pure label differentially private local randomizer, where $\varepsilon > 0$ is the privacy parameter. Smaller values of $\varepsilon$ provide stronger privacy guarantees, but introduce more noise.
\end{definition}

Instead of directly observing the true binary preference $y \in \{-1,1\}$ at each round, the learning agent  receives a privatized label $z \in \{-1,1\}$ obtained via randomized response (RR):
\begin{align}
    \mathbb{P}(z = y) 
    &= \alpha 
    := \frac{e^{\varepsilon}}{e^{\varepsilon} + 1}\in (0.5,1), \nonumber\\ 
    \mathbb{P}(z \neq y) 
    &= 1 - \alpha
    =\frac{1}{e^{\varepsilon} + 1}~. \label{eq:RR}
\end{align}
The above randomized response mechanism satisfies $\varepsilon$-pure local label DP~\citep{dwork2014algorithmic}.

\section{OFFLINE PRIVATE KL-REGULARIZED RLHF WITH PESSIMISM}
\label{sec:offline}

In this section, we will study the locally private KL-regularized RLHF in the offline setting. We will first provide the algorithm for the problem and derive its suboptimality upper bound. In order to show the optimality of the theoretical guarantee, we will also present the lower bound under the same assumptions.

\subsection{Algorithm and Upper Bound}
\begin{algorithm}[b]
\caption{Private Pessimistic KL-Regularized  RLHF (PPKL-RLHF) for Offline Setting}
\label{algo:offline}
\begin{algorithmic}[1]
\Require Regularization parameter $\beta$, reference policy $\pi_{\text{ref}}$, function class $\mathcal{F}$, offline dataset $\widetilde{\mathcal{D}} = \{(s_i, a_i^1, a_i^2, z_i)\}_{i=1}^n$
\State Compute the private MLE estimation of the reward function:
\[
\bar{r} \in \arg\max_{r \in \mathcal{F}} \ \sum_{i=1}^n \log \widetilde{P}_{r}(z_i \mid s_i, a_i^1, a_i^2)
\]
\State Use pessimism: $\hat{r}(s,a) \gets \bar{r}(s,a) - \Gamma_n(s,a), \forall (s,a)$, where $\Gamma_n$ is the bonus term in \eqref{eq:bonus}\\
\Return $\hat{\pi}(a \mid s) \propto \pi_{\text{ref}}(a \mid s) \exp\left( \beta \cdot \hat{r}(s, a) \right)$
\end{algorithmic}
\end{algorithm}


The main idea of Algorithm \ref{algo:offline} is that we first take the precollected data set $\widetilde{\mathcal{D}} = \{(s_i, a_i^1, a_i^2, z_i)\}_{i=1}^n$, where $z_i \in \{-1, +1\}$ are the privatized version of the true (unobserved) preference label $y_i$ through the randomized response mechanism in \eqref{eq:RR} with flip probability $1-\alpha$. For each sample $(s, a^1, a^2, z)$, the probability of private label $z$ given $s, a^1, a^2$ is 

\begin{equation} \label{eq:PrivProbTrue}
    \widetilde{P}_{r^*}(z \mid s, a^1, a^2)
    :=\mathbb{P}(z|s,a^1,a^2) 
    = \alpha \cdot \sigma(z \cdot \Delta_{r^*}(s, a^1, a^2)) 
    + (1 - \alpha) \cdot \sigma(-z \cdot \Delta_{r^*}(s, a^1, a^2)),
\end{equation}
where $\Delta_{r^*}(s, a^1, a^2) := r^*(s, a^1) - r^*(s, a^2)$ and $\sigma(x) = (1 + e^{-x})^{-1}$ is the sigmoid function. Building on the probability function  

\begin{equation}\label{eq:PrivProb}
    \widetilde{P}_{r}(z \mid s, a^1, a^2) 
     =  \alpha \cdot \sigma(z \cdot \Delta_{r}(s, a^1, a^2)) 
     + (1 - \alpha) \cdot \sigma(-z \cdot \Delta_{r}(s, a^1, a^2))
\end{equation}
of $z$ as a function of the reward $r$, we can estimate the reward by the Maximum Likelihood Estimation (MLE) on $\widetilde{P}_{r}(z \mid s, a^1, a^2)$ in step 1 of the algorithm. After we get the estimation of the reward $\bar{r}$, we construct a pessimistic estimator $\hat{r}$ in step 2 with the following value of the bonus $\Gamma_n(s,a)$:
\begin{equation}\label{eq:bonus}
     \sqrt{D^2_{\mathcal{F}}((s, a) ; \pi_{\text{ref}})\frac{c \cdot e^B}{(2\alpha-1)^2}\left(\frac{\log(\mathcal{N}_{\mathcal{F}}(\tau)/\delta)}{n}+ \tau \right)},
\end{equation}
where $c$ is a constant. 
Finally, we get the policy output by Gibbs distribution from \eqref{eq:PolicyImprove} based on $\hat{r}$.

\begin{remark}
    The pessimism principle is well-known in offline RL~\citep{jin2022pessimismprovablyefficientoffline} and offline RLHF~\citep{zhao2024sharp}. It consists in adopting the lower confidence bound of the reward estimation, since the conservative estimate helps the distributional shift challenge in the offline setting. In our local DP case, the main difference compared with the non-private case is that the effective sample size changes from $n$ to $(2\alpha-1)^2\cdot n=\left[{(e^\epsilon-1)}/{(e^\epsilon+1)}\right]^2 \cdot n < n$ due to the randomness from the privacy-preserving mechanism. 
\end{remark}


We now provide the theoretical guarantee of the suboptimality gap for the output policy in Algorithm~\ref{algo:offline}. We defer its detailed proof in Appendix B. 
\begin{theorem}[Sub-optimality gap upper bound in offline setting]\label{thm:offlineUpper}
    Under Assumptions~\ref{Assum:BT} and \ref{Assum:Realize}, Definitions~\ref{def:Ddivergence}, \ref{def:singleConcen} \ref{def:DensityRatio}, and \ref{def:Net}, for $\epsilon>0, \beta>0$ and a sufficiently small $\tau \in (0,1)$, with probability at least $1-\delta$, we have that the suboptimality gap of the output of Algorithm \ref{algo:offline}, $\text{SubOpt}(\hat{\pi})$ is of the order of
    \begin{equation}
          O\left(\beta D_{\pi^*}^2  \frac{e^B}{(2\alpha-1)^2}\left(\frac{\log(\mathcal{N}_{\mathcal{F}}(\tau)/\delta)}{n}+ \tau \right) \right).
    \end{equation}
\end{theorem}

\textbf{Proof sketch:} We first show that the suboptimality gap is upper-bounded by the reward model estimation error:
 \[
 \text{SubOpt}(\hat{\pi})
        \le \beta \cdot D_{\pi^*}^2  \cdot \text{Err}_{RM},
\]
where
\[
\text{Err}_{RM}=\mathbb{E}_{(s,a)\sim d_0\times \pi_{\text{ref}}}[(\bar{r}(s,a)-b(s)-r^*(s,a))^2],
\] and $\bar{r}$ is the private reward estimation from step 1 in Algorithm \ref{algo:offline} and $b(s)$ is a bias function of $s$. This also provides the key takeaway in RLHF: the policy performance depends on the reward model. 
Then, we focus on the on-policy error bound of the reward estimation and derive it by Ville’s inequality and Freedman's Inequality. Building on the confidence bound of reward estimation, we design the bonus for the pessimistic principle.

\begin{remark}[Discussion of the parameters in the upper bound]
    In the above results, $\beta$ is a hyperparameter in the regularized objective function \eqref{eq:ReguObj} to trade off the reward maximization and how close the target policy is to $\pi_\text{ref}$. $e^B$ comes from the sigmoid function in BT preference model and it is common in the RLHF literature \citep{zhou2025unifiedtheoreticalanalysisprivate,xiong2023iterative,zhao2024sharp,zhao2025nearly}. 
\end{remark}

\begin{remark}[Comparision with prior work for upper bound]
    Compared with the unregularized suboptimality upper bound of $\widetilde{O}(1/[(2\alpha-1)\sqrt{n}]) $ in \cite{zhou2025unifiedtheoreticalanalysisprivate} with their single-policy relative condition number, our result with KL-regularization of $\widetilde{O}(1/[(2\alpha-1)^2 {n} ])$ is tighter when the sample size $n$ is large enough, but on a different objective function. When $\epsilon \in (0,1]$, which means a strong privacy guarantee, we obtain $\widetilde{O}(1/[(2\alpha-1)^2 {n} ])= \widetilde{O}(1/[(e^\epsilon-1)^2 n])$ that matches the lower bound we prove in the following. Note that when $\epsilon\rightarrow +\infty$, i.e., $\alpha = 1$, we recover the non-private case in  \citet{zhao2025sharpanalysisofflinepolicy}. 
\end{remark}

\subsection{Lower Bound Analysis}
We verify the optimality of the above bound by proving the following lower bound and defer the complete proof to Appendix B.
\begin{theorem}[Sub-optimality gap lower bound in offline setting]\label{thm:offlineLower}
For reward function class $\mathcal{F} \subset(\mathcal{S} \times \mathcal{A} \to [0,B])$, $\tau \in (0,1)$ small enough, $\beta>0$, $ S=\log\mathcal{N}_\mathcal{F}(\tau)$, $C^* \in (2, e^{(\beta B)/2}+1)$, algorithm set $\Pi$, $C^{\pi^*}\le C^*$, and KL-regularized RLHF instance set $\mathcal{I}$, the minimax suboptimality gap $\inf_{\hat{\pi}\in \Pi} \sup_{I \in \mathcal{I}} \ \text{SubOpt}(\hat{\pi},I) $ under $\epsilon$-LDP mechanism for labels is 
\begin{equation}
   \Omega\left(\min\left\{\frac{\beta C^* \log\mathcal{N}_\mathcal{F}(\tau)}{(e^\epsilon-1)^2 n},\frac{\sqrt{\log\mathcal{N}_\mathcal{F}(\tau)C^*}}{(e^\epsilon-1)\sqrt{n}}\right\}\right)~.
\end{equation}
\end{theorem}

\textbf{Proof sketch:} We summarize our proof as follows:
\begin{itemize}
    \item \textit{Step 1:} First, we construct a family of instances indexed by the hypercube $\{-1,+1\}^S$. For each state 
$s$, set rewards so that the KL-regularized optimal policy chooses the actions based on Equation \eqref{eq:PolicyImprove}, and we verify the single-policy coverage based on the construction.
\item \textit{Step 2:} We equate the suboptimality gap of each instance by the KL divergence between the estimated policy and the optimal policy and then construct a hard-to-distinguish pair. 
\item \textit{Step 3:} Finally, we apply a KL-divergence inequality under LDP from Theorem 1 in \cite{duchi2013local} for the label distribution and a variant of the (private) version of Assouad’s lemma on the hypercube to get the minimax suboptimality lower bound.
\end{itemize}

 \begin{remark}[Comparision with prior work for lower bound]
    A lower bound for the \emph{parameter estimation} for RLHF under label LDP is provided in \citet{chowdhury2024differentially}. In particular, they show a lower bound of $\Omega(\frac{1}{e^{\varepsilon}-1} \sqrt{\frac{d}{n}})$ for the estimation error bound of the parameter in a linear reward model in $\mathbb{R}^d$. As far as we know, we are the first ones to provide the lower bound for the \emph{suboptimality gap} for this problem of RLHF under LDP, matching the same effective sample size of $(e^\epsilon-1)^2 n \approx \epsilon^2 n$ when $\epsilon \in (0,1)$ as \citet{chowdhury2024differentially}. Taking $\mathcal{N}_\mathcal{F}(\tau)=(1/\tau)^d$ in the linear model, we can imply the suboptimality gap of 
    $\widetilde{\Omega}\left(\min\left\{\frac{\beta C^* d}{(e^\epsilon-1)^2 n},\frac{\sqrt{dC^*}}{(e^\epsilon-1)\sqrt{n}}\right\}\right)$ for private KL-regularized RLHF which also demonstrates the importance of $\beta$ and $C^*$ in this problem.
\end{remark}

\begin{remark}[Discussion of the parameters in the lower bound]
    From the above lower bound and the upper bound of the suboptimality gap in Theorem \ref{thm:offlineUpper}, we obtain that the single-policy coverage $C^{\pi^*}$ is necessary due to the distribution shift between the behavior policy and optimal in the private RLHF problem. In fact, \cite{foster2025good} showed that in the non-private RLHF setting the single policy coverage coefficient is also unavoidable. Motivated by this, in the next section we study the problem of private KL-regularized RLHF under an online setting, which will help remove the dependence on the coverage condition.
\end{remark}


\section{ONLINE PRIVATE KL-REGULARIZED RLHF WITH OPTIMISM}
\label{sec:online}
In this section, we turn our attention to KL-Regularized RLHF with LDP on labels in the online setting. Compared with the online RL problem, the main challenge of online RLHF comes from the imperfect information on the reward. That is, the reward can be observed in RL and used to estimate the reward model. However, in online RLHF, given a context, we need to sample two actions and receive human labels to train the reward model. This raises another problem: How to sample two actions?

The sampling methods of two actions in online RLHF are mainly divided into two classes: symmetric and non-symmetric.
\begin{itemize}
    \item In the symmetric class, we sample two actions from the same policy, e.g., the one got from the last iteration as in \citet{cen2024value,guo2024direct}. However, \citet[Proposition 2.1]{xie2024exploratory} shows that this strategy can suffer from a constant lower bound on the suboptimality gap. Hence, some kind of exploration is necessary in online RLHF.
    \item In the non-symmetric class, some algorithms sample actions from different polices---one policy from exploitation and another one for exploration based on the first one---for KL regularized RLHF~\citep{xiong2024building,xiong2023iterative}. \citet{xie2024exploratory,chen2025outcome} sample an action from the last iteration policy and another from the reference policy for  KL regularized  RLHF, but adds a bias term in the loss function for exploration. 
\end{itemize} 

Inspired by the above works, we adopt the optimism principle for our exploration policy, which is a principle widely used in online RL~\citep{xiong2023sufficient,moulin2025optimistically,moulin2023optimistic,zhao2025logarithmic}.
We develop the Private Optimistic KL-Regularized
RLHF (POKL-RLHF) algorithm (see Algorithm \ref{algo:online}). In each time step $t\in \{1,\dots,T\}$, after the learner observes the context $s_t$ (the prompt in the large language model) sampled from a fixed distribution $d_0$, two actions (two answers from the LLM) are compared. In our LDP model, only the private label $z_i$ privatized by the RR mechanism in \eqref{eq:RR} is available to the learner, instead of the true label $y_i$. With these historical data till time step $t$, we update the reward model by the private least squares estimation at Step 7. Then, we update the exploitation policy $\pi^1_{t+1}$ based on the reward estimation by the solution of the KL-regularized objective function in \eqref{eq:PolicyImprove}. Given $\pi^1_{t+1}$, we design the exploration policy by using an exploration bonus. In particular, we construct a confidence set that will shrink with time:
\[
    \mathcal{F}_t
    =\left\{r \in \mathcal{F}: \sum_{i=1}^t\left( \Delta_i^r- \Delta_i^{\bar{r}_t}\right)^2+\lambda \leq \Gamma_T^2\right\},
\]
where
\[
\Gamma_T=\frac{c e^{B} \sqrt{\log \left(T \cdot N_{\mathcal{F}}  / \delta\right)}}{2\alpha-1}
\]
and $c$ is a constant. Then, the exploration bonus $b_t$ is defined through the uncertainty in Definition~\ref{def:Uncert}:
\begin{equation}\label{eq:ExploreBounus}
     b_t(s, a)
     =\min \left\{1, \Gamma_T \, U_{\mathcal{F}_t}\left(\lambda, s, a ; \mathcal{D}_{t};\pi_{t+1}^1\right)\right\}.
\end{equation}

\begin{algorithm}[t]
\caption{ Private Optimistic KL-Regularized  RLHF (POKL-RLHF) for Online Setting}
\label{algo:online}
\begin{algorithmic}[1]
\Require KL coefficient $\beta$, reward function class $\mathcal{F}$, exploration scale $\lambda$, reference policy $\pi_{\text{ref}}$, DP parameter $\varepsilon$
\State \textbf{Initialize:} $\mathcal{D}_0 = \emptyset$; $\pi_1^1,\pi_1^2 = \pi_{\text{ref}}$
\For{$t = 1$ to $T$}
    \State Observe context $s_t \sim d_0$
    \State Sample $a_t^1 \sim \pi_t^1(\cdot \mid s_t)$ and $a_t^2 \sim \pi_t^2(\cdot \mid s_t)$
    \State Observe private preference label $z_t \in \{-1, 1\}$ via randomized response in \eqref{eq:RR}
    \State Update $\mathcal{D}_{t} \gets \mathcal{D}_{t-1} \cup \{(s_t, a_t^1, a_t^2, z_t)\}$
    \State Estimate reward from private least square:
    \[
       \bar{r}_t = \arg\min_{r \in \mathcal{F}} \ \sum_{ \mathcal{D}_{t}} \left[ {(2\sigma(\Delta_i^r) - 1)(2\alpha-1)} - z_i \right]^2,
    \]
    where  $\Delta_i^r:= r( s_i, a_i^1) - r(s_i, a_i^2)$
    \State Update exploitation policy: $\pi_{t+1}^1(a \mid s) \propto \pi_{\text{ref}}(a \mid s) \cdot \exp(\beta \cdot \bar{r}_t(s,a))$
    \State Set exploration policy: $\pi_{t+1}^2(a \mid s) \propto \pi_{t+1}^1(a \mid s) \cdot \exp(\beta \cdot b_t(s,a))$ with $b_t$ defined in \eqref{eq:ExploreBounus}
\EndFor
\end{algorithmic}
\end{algorithm}

\begin{remark}
  As in \citet{huang2025correctingmythosklregularizationdirect,zhao2025logarithmic}, we assume that the reward function space $\mathcal{F}$ is finite. The infinite case can be solved easily by an $\epsilon$-net and uniform convergence argument (refer to Lemma C.1 in \citet{zhao2025logarithmic} and Lemma C.2 in \citet{zhao2024sharp}), similarly to our offline case.
\end{remark}

Based on the optimism principle for exploration policy, we derive the following theoretical guarantee.
\begin{theorem}[Regret Bound]\label{thm:onlineRegret}
    Under Assumptions~\ref{Assum:BT} and \ref{Assum:Realize}, for $\delta\in (0,1), \epsilon>0$ and $\lambda \le \frac{1}{2} \Gamma_T^2$ with probability at least $1-\delta$, Algorithm~\ref{algo:online} satisfies 
    \[
    \sum_{t=1}^T (J(\pi^*) -J(\pi^2_t)) 
    = O\left( \frac{\beta \cdot d_{\mathcal{F} }\cdot e^{2B}}{(2\alpha-1)^2} \log(N_{\mathcal{F}}\cdot T/\delta) \right),
    \]
 where $d_\mathcal{F}$ is the pair eluder dimension in Definition~\ref{def:Uncert}, $\beta$ is the hyperparameter in \eqref{eq:ReguObj}, $N_{\mathcal{F}}$ is the cardinality of reward function space.
\end{theorem}

\begin{remark}
\label{sec:remark_5_3}
    In the context of online RL/RLHF, bounds in terms of the eluder dimension characterize the statistical learnability of exploration strategies. However, it is important to note that such guarantees are information-theoretic rather than computational: While they demonstrate that learning is possible with a finite number of iterations, the corresponding algorithms are often computationally intractable when the function class is large. 
    We leave how to find a computationally efficient method with logarithmic regret for online RLHF as an open problem.
\end{remark}

\begin{remark}
    In the above results, $ e^{2B}$ comes from the sigmoid function for the preference model. The effect of LDP is a factor of $\frac{1}{(2\alpha-1)^2}=(\frac{e^\epsilon+1}{e^\epsilon-1})^2 >1$ due to the randomness from the differential privacy mechanism. As a by-product, taking $\epsilon \rightarrow +\infty$, i.e., $\alpha=1$ in the algorithm analysis, the result implies a bound for the corresponding non-private case.
\end{remark}

\begin{corollary}
Under Assumptions~\ref{Assum:BT} and \ref{Assum:Realize}, for $\alpha=1$, $\delta\in (0,1)$, with probability at least $1-\delta$, Algorithm~\ref{algo:online} satisfies
\[
\sum_{t=1}^T (J(\pi^*) -J(\pi^2_t)) 
= O\left( {\beta \cdot d_{\mathcal{F} }\cdot e^{2B}}\log(N_{\mathcal{F}}\cdot T/\delta) \right)~.
\]
\end{corollary}

\begin{remark}
    Online RLHF is also studied in  \citep[Section 4]{xiong2023iterative}, and from their proofs a sublinear regret bound of $\tilde{O}(\sqrt{T})$ for the exploration policy can be implied. Compared with their results, we are the first ones to achieve a logarithmic regret bound with the eluder dimension.
\end{remark}

\section{EXPERIMENTAL RESULTS}


As noted in Remark \ref{sec:remark_5_3}, the online algorithm based on the eluder dimension is computationally intractable in practice. Thus, we choose to only experiment in the offline case to empirically verify our theoretical findings about the effect of the $\epsilon$-LDP model.

\noindent\textbf{Dataset and Compute} For all experiments, we use the helpful assistant preference corpus\footnote{\url{https://huggingface.co/datasets/Anthropic/hh-rlhf}} tailored for RLHF~\citep{bai2022training}. The dataset consists of two complementary components: (i) Supervised Fine-Tuning (SFT) dialogues, where each sample contains a user query and a preferred assistant response; and (ii) preference pairs, where each sample provides a prompt together with one chosen and one rejected response. The SFT corpus contains $38{,}821$ training examples and $4{,}413$ validation examples. Preference pairs are split into $38{,}821$ training, $2{,}100$ validation, and $2{,}313$ held-out test examples.
We used a single AMD MI-200 GPU equipped with 64 GB of VRAM.  

\noindent\textbf{SFT training and Baseline} We use the Llama-3.2-1B-Instruct model\footnote{\url{https://huggingface.co/meta-llama/Llama-3.2-1B-Instruct}} as the backbone for all experiments. To obtain the baseline policy $\pi_0$, we performed SFT on the dialogue part of the dataset, with standard next-token prediction.

We also use Direct Preference Optimization (DPO)~\citep{rafailov2023direct} as a baseline, training the policy relative to the frozen SFT reference $\pi_0$ on the preference pairs. The objective is optimized for $\beta=0.1$ with AdamW, linear warmup, gradient accumulation, and validation every 500 steps, and the best checkpoint is selected by validation loss after a few thousand iterations. This baseline is non-private and without KL regularization.

\noindent\textbf{Implementation of PPKL-RLHF
} To implement this setup we first train a privatized reward model (Algorithm~\ref{algo:offline}) that adds a scalar linear head with EOS pooling on top of the Llama-3.2-1B-Instruct backbone, clipped to $[-5,5]$. The reward model is optimized in two phases: first warming up by training only the head, then fine-tuning the full backbone for 5 epochs. 
The policy is optimized with PPO~\citep{schulman2017proximalpolicyoptimizationalgorithms} against the corrected rewards and a KL penalty to the SFT baseline, using $\beta=0.1$. Training runs for $500$ iterations with $16$ rollouts per iteration; each update applies $3$ PPO epochs with minibatch size $4$, generation length capped at $64$ tokens (prompts up to $256$, temperature $1.0$, top-p $0.9$), and standard PPO hyperparameters (clip $\epsilon_c=0.2$, policy lr $1\times10^{-6}$, value lr $5\times10^{-6}$, value loss weight $0.5$, entropy coefficient $0.01$, max grad norm $1.0$). 
\begin{figure}[t]
    \centering
    \includegraphics[width=\linewidth]{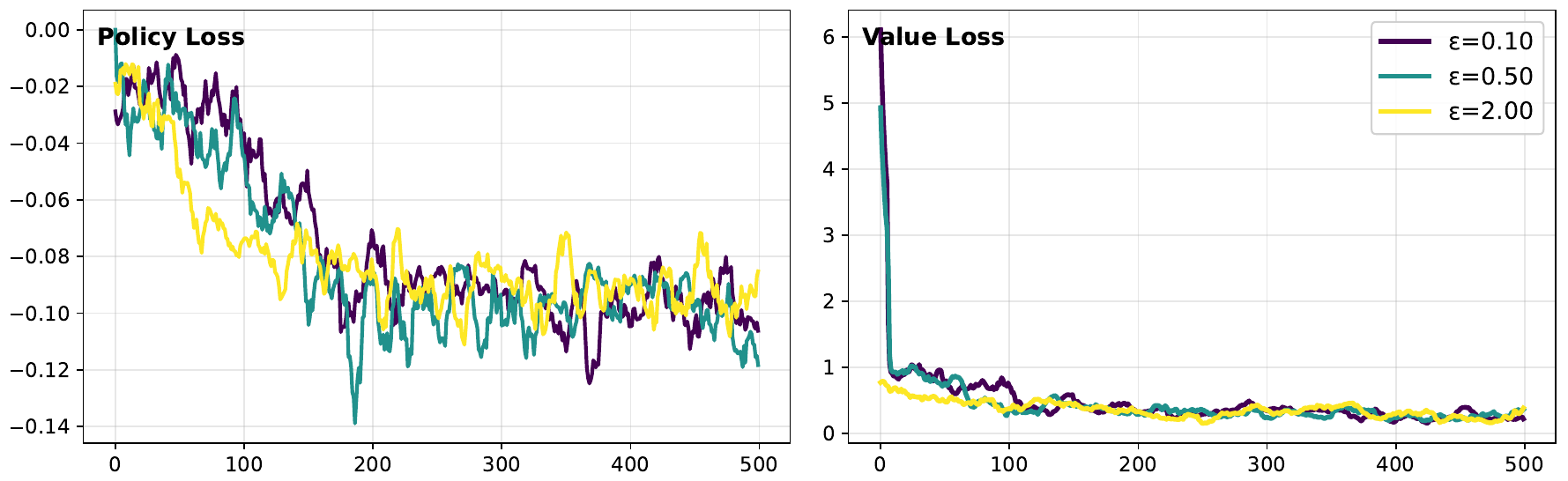}
    \caption{Training metrics for our Private KL-Regularized RLHF over iterations for different $\epsilon$ vals.}
    \label{fig:training_curves}
\end{figure}

\noindent\textbf{Training Performance} 
In Figure~\ref{fig:training_curves}, we track two core metrics of training. The policy loss, also used in \cite{schulman2017proximal},
\[
-\mathbb{E}_t \Big[\min(r_t(\theta)A_t,\;\text{clip}(r_t(\theta),1-\epsilon_c,1+\epsilon_c)A_t )\Big],
\] 
measures how effectively the new policy improves while keeping updates stable, where the advantage function,
$
A_t = R_t - V_\phi(s_t)
$,
quantifies the relative gain of an action compared to the baseline value function, $ r_t(\theta)$  denote the probability ratio $ r_t(\theta)=\frac{\pi_\theta\left(a_t \mid s_t\right)}{\pi_{\theta_{\text {old }}}\left(a_t \mid s_t\right)}$, $\hat{\mathbb{E}}_t$  indicates the empirical average and $t$ is the iteration index. The value loss,
$
\mathbb{E}\big[(V_\phi(s)-R)^2\big],
$
evaluates how accurately the value function predicts expected returns.

As showcased in Figure~\ref{fig:training_curves}, the policy loss decreases steadily and converges to a low plateau, while the value loss drops sharply before stabilizing. As privacy is relaxed, both metrics improve. At $\epsilon=0.10$, the policy loss remains relatively high and the value loss bottoms out at $0.072$. At $\epsilon=0.50$, both show stronger improvement, with the value loss converging to a lower value. At $\epsilon=2.00$, training achieves the best utility: policy loss decreases most rapidly and value loss reaches its lowest point ($0.062$). These results confirm that higher $\epsilon$ (weaker privacy) yields stronger learning signals and more effective optimization, showing the expected trade-off between performance and privacy.

\begin{table}[t]
\centering
\caption{Win rates of different methods evaluated on the preference test set. PPKL-RLHF uses $\beta=0.10$.}
\label{tab:winrates}
\begin{tabular}{lcc}
\toprule
\textbf{Method} & \textbf{Setting} & \textbf{Win rate} \\
\midrule
SFT ($\pi_{0}$) & --        & 0.538 \\
DPO(non-private)             & $\beta=0.1$ & 0.704 \\
\hline
PPKL-RLHF       & $\epsilon=0.1$ & 0.530 \\
PPKL-RLHF       & $\epsilon=0.5$ & 0.554 \\
PPKL-RLHF       & $\epsilon=2.0$ & 0.607 \\
\bottomrule
\end{tabular}
\end{table}
\noindent\textbf{Results and Baseline Comparison} The final results of our evaluation are presented in Table~\ref{tab:winrates} where we use the win rate as our performance metric, as in \citet{rafailov2023direct,zhou2025unifiedtheoreticalanalysisprivate}. At stronger privacy ($\epsilon{=}0.1$) performance is close to SFT, while at $\epsilon{=}0.5$ it surpasses the SFT baseline ($0.554$ vs.\ $0.538$). The best setting reaches around $0.607$ at $\epsilon{=}2.0$, indicating utility gains with weaker theoretical privacy. These results highlight that even with noisy privatized labels, training a reward model followed by our PPKL-RLHF procedure retains competitive utility and offers tunable privacy–utility trade-offs. However, PPKL-RLHF's win-rate remains behind DPO ($0.704$), likely because label privatization and the pessimistic KL correction restrict the effective learning signal compared to the non-private baseline. Achieving performance closer to the non-private DPO baseline remains an open direction for future work.

\section{CONCLUSION}

In this paper, we investigated the KL-regularized RLHF problem in both offline and online settings. We designed algorithms based on pessimistic and optimistic principles for the offline and online settings, respectively, and provided theoretical guarantees for both cases.
We established the optimal sub-optimality gaps for the offline setting and a logarithmic regret bound for the online setting while preserving privacy. 
Finally, we also showed some experimental results to verify our theoretical findings.

\section{Acknowledgement}
We thank Wei Xiong, Xingyu Zhou, and Yuhui Wang for insightful discussions. We would like to acknowledge the MBZUAI SU Fund and MIT–MBZUAI Collaborative Research Program for supporting this work.


\bibliographystyle{plainnat}
\bibliography{ref}

\begin{thebibliography}{44}
\providecommand{\natexlab}[1]{#1}
\providecommand{\url}[1]{\texttt{#1}}
\expandafter\ifx\csname urlstyle\endcsname\relax
  \providecommand{\doi}[1]{doi: #1}\else
  \providecommand{\doi}{doi: \begingroup \urlstyle{rm}\Url}\fi

\bibitem[Agarwal et~al.(2023)Agarwal, Jin, and Zhang]{agarwal2023vo}
Alekh Agarwal, Yujia Jin, and Tong Zhang.
\newblock Vo $ q $ l: Towards optimal regret in model-free rl with nonlinear function approximation.
\newblock In \emph{The Thirty Sixth Annual Conference on Learning Theory}, pages 987--1063. PMLR, 2023.

\bibitem[Aminian et~al.(2025)Aminian, Asadi, Shenfeld, and Mroueh]{aminian2025theoretical}
Gholamali Aminian, Amir~R Asadi, Idan Shenfeld, and Youssef Mroueh.
\newblock Theoretical analysis of kl-regularized rlhf with multiple reference models.
\newblock \emph{arXiv preprint arXiv:2502.01203}, 2025.

\bibitem[Bai et~al.(2022)Bai, Jones, Ndousse, Askell, Chen, DasSarma, Drain, Fort, Ganguli, Henighan, et~al.]{bai2022training}
Yuntao Bai, Andy Jones, Kamal Ndousse, Amanda Askell, Anna Chen, Nova DasSarma, Dawn Drain, Stanislav Fort, Deep Ganguli, Tom Henighan, et~al.
\newblock Training a helpful and harmless assistant with reinforcement learning from human feedback.
\newblock \emph{arXiv preprint arXiv:2204.05862}, 2022.

\bibitem[Cen et~al.(2024)Cen, Mei, Goshvadi, Dai, Yang, Yang, Schuurmans, Chi, and Dai]{cen2024value}
Shicong Cen, Jincheng Mei, Katayoon Goshvadi, Hanjun Dai, Tong Yang, Sherry Yang, Dale Schuurmans, Yuejie Chi, and Bo~Dai.
\newblock Value-incentivized preference optimization: A unified approach to online and offline rlhf.
\newblock \emph{arXiv preprint arXiv:2405.19320}, 2024.

\bibitem[Chen et~al.(2025)Chen, Jia, Rakhlin, and Xie]{chen2025outcome}
Fan Chen, Zeyu Jia, Alexander Rakhlin, and Tengyang Xie.
\newblock Outcome-based online reinforcement learning: Algorithms and fundamental limits.
\newblock \emph{arXiv preprint arXiv:2505.20268}, 2025.

\bibitem[Chen et~al.(2022)Chen, Zhong, Yang, Wang, and Wang]{chen2022human}
Xiaoyu Chen, Han Zhong, Zhuoran Yang, Zhaoran Wang, and Liwei Wang.
\newblock Human-in-the-loop: Provably efficient preference-based reinforcement learning with general function approximation.
\newblock In \emph{International Conference on Machine Learning}, pages 3773--3793. PMLR, 2022.

\bibitem[Chowdhury et~al.(2024)Chowdhury, Zhou, and Natarajan]{chowdhury2024differentially}
Sayak~Ray Chowdhury, Xingyu Zhou, and Nagarajan Natarajan.
\newblock Differentially private reward estimation with preference feedback.
\newblock In \emph{International Conference on Artificial Intelligence and Statistics}, pages 4843--4851. PMLR, 2024.

\bibitem[Duchi et~al.(2013)Duchi, Jordan, and Wainwright]{duchi2013local}
John~C Duchi, Michael~I Jordan, and Martin~J Wainwright.
\newblock Local privacy and statistical minimax rates.
\newblock In \emph{2013 IEEE 54th annual symposium on foundations of computer science}, pages 429--438. IEEE, 2013.

\bibitem[Dwork et~al.(2014)Dwork, Roth, et~al.]{dwork2014algorithmic}
Cynthia Dwork, Aaron Roth, et~al.
\newblock The algorithmic foundations of differential privacy.
\newblock \emph{Foundations and Trends{\textregistered} in Theoretical Computer Science}, 9\penalty0 (3--4):\penalty0 211--407, 2014.

\bibitem[Foster et~al.(2021)Foster, Kakade, Qian, and Rakhlin]{foster2021statistical}
Dylan~J Foster, Sham~M Kakade, Jian Qian, and Alexander Rakhlin.
\newblock The statistical complexity of interactive decision making.
\newblock \emph{arXiv preprint arXiv:2112.13487}, 2021.

\bibitem[Foster et~al.(2025)Foster, Mhammedi, and Rohatgi]{foster2025good}
Dylan~J Foster, Zakaria Mhammedi, and Dhruv Rohatgi.
\newblock Is a good foundation necessary for efficient reinforcement learning? the computational role of the base model in exploration.
\newblock \emph{arXiv preprint arXiv:2503.07453}, 2025.

\bibitem[Ghazi et~al.(2021)Ghazi, Golowich, Kumar, Manurangsi, and Zhang]{ghazi2021deep}
Badih Ghazi, Noah Golowich, Ravi Kumar, Pasin Manurangsi, and Chiyuan Zhang.
\newblock Deep learning with label differential privacy.
\newblock \emph{Advances in neural information processing systems}, 34:\penalty0 27131--27145, 2021.

\bibitem[Guo et~al.(2024)Guo, Zhang, Liu, Liu, Khalman, Llinares, Rame, Mesnard, Zhao, Piot, et~al.]{guo2024direct}
Shangmin Guo, Biao Zhang, Tianlin Liu, Tianqi Liu, Misha Khalman, Felipe Llinares, Alexandre Rame, Thomas Mesnard, Yao Zhao, Bilal Piot, et~al.
\newblock Direct language model alignment from online ai feedback.
\newblock \emph{arXiv preprint arXiv:2402.04792}, 2024.

\bibitem[Huang et~al.(2025)Huang, Zhan, Xie, Lee, Sun, Krishnamurthy, and Foster]{huang2025correctingmythosklregularizationdirect}
Audrey Huang, Wenhao Zhan, Tengyang Xie, Jason~D. Lee, Wen Sun, Akshay Krishnamurthy, and Dylan~J. Foster.
\newblock Correcting the mythos of kl-regularization: Direct alignment without overoptimization via chi-squared preference optimization, 2025.
\newblock URL \url{https://arxiv.org/abs/2407.13399}.

\bibitem[Jin et~al.(2022)Jin, Yang, and Wang]{jin2022pessimismprovablyefficientoffline}
Ying Jin, Zhuoran Yang, and Zhaoran Wang.
\newblock Is pessimism provably efficient for offline rl?, 2022.
\newblock URL \url{https://arxiv.org/abs/2012.15085}.

\bibitem[Li et~al.(2025)Li, Qian, Zhao, and Zhou]{li2025provably}
Long-Fei Li, Yu-Yang Qian, Peng Zhao, and Zhi-Hua Zhou.
\newblock Provably efficient rlhf pipeline: A unified view from contextual bandits.
\newblock \emph{ArXiv preprint}, 2502, 2025.

\bibitem[Liu et~al.(2023)Liu, Yao, Ton, Zhang, Guo, Cheng, Klochkov, Taufiq, and Li]{liu2023trustworthy}
Yang Liu, Yuanshun Yao, Jean-Francois Ton, Xiaoying Zhang, Ruocheng Guo, Hao Cheng, Yegor Klochkov, Muhammad~Faaiz Taufiq, and Hang Li.
\newblock Trustworthy llms: a survey and guideline for evaluating large language models' alignment.
\newblock \emph{arXiv preprint arXiv:2308.05374}, 2023.

\bibitem[Moulin and Neu(2023)]{moulin2023optimistic}
Antoine Moulin and Gergely Neu.
\newblock Optimistic planning by regularized dynamic programming.
\newblock In \emph{International Conference on Machine Learning}, pages 25337--25357. PMLR, 2023.

\bibitem[Moulin et~al.(2025)Moulin, Neu, and Viano]{moulin2025optimistically}
Antoine Moulin, Gergely Neu, and Luca Viano.
\newblock Optimistically optimistic exploration for provably efficient infinite-horizon reinforcement and imitation learning.
\newblock \emph{arXiv preprint arXiv:2502.13900}, 2025.

\bibitem[Osband and Van~Roy(2014)]{osband2014model}
Ian Osband and Benjamin Van~Roy.
\newblock Model-based reinforcement learning and the eluder dimension.
\newblock \emph{Advances in neural information processing systems}, 27, 2014.

\bibitem[Rafailov et~al.(2023)Rafailov, Sharma, Mitchell, Manning, Ermon, and Finn]{rafailov2023direct}
Rafael Rafailov, Archit Sharma, Eric Mitchell, Christopher~D Manning, Stefano Ermon, and Chelsea Finn.
\newblock Direct preference optimization: Your language model is secretly a reward model.
\newblock \emph{Advances in neural information processing systems}, 36:\penalty0 53728--53741, 2023.

\bibitem[Russo and Van~Roy(2013)]{russo2013eluder}
Daniel Russo and Benjamin Van~Roy.
\newblock Eluder dimension and the sample complexity of optimistic exploration.
\newblock \emph{Advances in Neural Information Processing Systems}, 26, 2013.

\bibitem[Schulman et~al.(2017{\natexlab{a}})Schulman, Wolski, Dhariwal, Radford, and Klimov]{schulman2017proximal}
John Schulman, Filip Wolski, Prafulla Dhariwal, Alec Radford, and Oleg Klimov.
\newblock Proximal policy optimization algorithms.
\newblock \emph{arXiv preprint arXiv:1707.06347}, 2017{\natexlab{a}}.

\bibitem[Schulman et~al.(2017{\natexlab{b}})Schulman, Wolski, Dhariwal, Radford, and Klimov]{schulman2017proximalpolicyoptimizationalgorithms}
John Schulman, Filip Wolski, Prafulla Dhariwal, Alec Radford, and Oleg Klimov.
\newblock Proximal policy optimization algorithms, 2017{\natexlab{b}}.
\newblock URL \url{https://arxiv.org/abs/1707.06347}.

\bibitem[Su(2025)]{su2025largelanguagemodelsreally}
Weijie Su.
\newblock Do large language models (really) need statistical foundations?, 2025.
\newblock URL \url{https://arxiv.org/abs/2505.19145}.

\bibitem[Wang et~al.(2020)Wang, Salakhutdinov, and Yang]{wang2020reinforcement}
Ruosong Wang, Russ~R Salakhutdinov, and Lin Yang.
\newblock Reinforcement learning with general value function approximation: Provably efficient approach via bounded eluder dimension.
\newblock \emph{Advances in Neural Information Processing Systems}, 33:\penalty0 6123--6135, 2020.

\bibitem[Wu and Sun(2023)]{wu2023making}
Runzhe Wu and Wen Sun.
\newblock Making rl with preference-based feedback efficient via randomization.
\newblock \emph{arXiv preprint arXiv:2310.14554}, 2023.

\bibitem[Xie et~al.(2022)Xie, Foster, Bai, Jiang, and Kakade]{xie2022role}
Tengyang Xie, Dylan~J Foster, Yu~Bai, Nan Jiang, and Sham~M Kakade.
\newblock The role of coverage in online reinforcement learning.
\newblock \emph{arXiv preprint arXiv:2210.04157}, 2022.

\bibitem[Xie et~al.(2024)Xie, Foster, Krishnamurthy, Rosset, Awadallah, and Rakhlin]{xie2024exploratory}
Tengyang Xie, Dylan~J Foster, Akshay Krishnamurthy, Corby Rosset, Ahmed Awadallah, and Alexander Rakhlin.
\newblock Exploratory preference optimization: Harnessing implicit q*-approximation for sample-efficient rlhf.
\newblock \emph{arXiv preprint arXiv:2405.21046}, 2024.

\bibitem[Xiong(2023)]{xiong2023sufficient}
Wei Xiong.
\newblock A sufficient condition of sample-efficient reinforcement learning with general function approximation.
\newblock \emph{The Hong Kong University of Science and Technology}, 2023.

\bibitem[Xiong et~al.(2023)Xiong, Dong, Ye, Wang, Zhong, Ji, Jiang, and Zhang]{xiong2023iterative}
Wei Xiong, Hanze Dong, Chenlu Ye, Ziqi Wang, Han Zhong, Heng Ji, Nan Jiang, and Tong Zhang.
\newblock Iterative preference learning from human feedback: Bridging theory and practice for rlhf under kl-constraint.
\newblock \emph{arXiv preprint arXiv:2312.11456}, 2023.

\bibitem[Xiong et~al.(2024)Xiong, Shi, Shen, Rosenberg, Qin, Calandriello, Khalman, Joshi, Piot, Saleh, et~al.]{xiong2024building}
Wei Xiong, Chengshuai Shi, Jiaming Shen, Aviv Rosenberg, Zhen Qin, Daniele Calandriello, Misha Khalman, Rishabh Joshi, Bilal Piot, Mohammad Saleh, et~al.
\newblock Building math agents with multi-turn iterative preference learning.
\newblock \emph{arXiv preprint arXiv:2409.02392}, 2024.

\bibitem[Ye et~al.(2023)Ye, Xiong, Gu, and Zhang]{ye2023corruption}
Chenlu Ye, Wei Xiong, Quanquan Gu, and Tong Zhang.
\newblock Corruption-robust algorithms with uncertainty weighting for nonlinear contextual bandits and markov decision processes.
\newblock In \emph{International Conference on Machine Learning}, pages 39834--39863. PMLR, 2023.

\bibitem[Ye et~al.(2024)Ye, Xiong, Zhang, Dong, Jiang, and Zhang]{ye2024online}
Chenlu Ye, Wei Xiong, Yuheng Zhang, Hanze Dong, Nan Jiang, and Tong Zhang.
\newblock Online iterative reinforcement learning from human feedback with general preference model.
\newblock \emph{Advances in Neural Information Processing Systems}, 37:\penalty0 81773--81807, 2024.

\bibitem[Zhang(2023)]{zhang2023math}
Tong Zhang.
\newblock \emph{Mathematical analysis of machine learning algorithms}.
\newblock Cambridge University Press, 2023.

\bibitem[Zhang et~al.(2025{\natexlab{a}})Zhang, Panaganti, Shi, Ziani, and Wierman]{zhang2025klregularizationdifferentiallyprivatebandits}
Yizhou Zhang, Kishan Panaganti, Laixi Shi, Juba Ziani, and Adam Wierman.
\newblock Kl-regularization itself is differentially private in bandits and rlhf, 2025{\natexlab{a}}.
\newblock URL \url{https://arxiv.org/abs/2505.18407}.

\bibitem[Zhang et~al.(2025{\natexlab{b}})Zhang, Sun, Yang, Cui, Wang, and Huang]{zhang2025carefulfinetuningopensourcellms}
Zhexin Zhang, Yuhao Sun, Junxiao Yang, Shiyao Cui, Hongning Wang, and Minlie Huang.
\newblock Be careful when fine-tuning on open-source llms: Your fine-tuning data could be secretly stolen!, 2025{\natexlab{b}}.
\newblock URL \url{https://arxiv.org/abs/2505.15656}.

\bibitem[Zhao et~al.(2024)Zhao, Ye, Gu, and Zhang]{zhao2024sharp}
Heyang Zhao, Chenlu Ye, Quanquan Gu, and Tong Zhang.
\newblock Sharp analysis for kl-regularized contextual bandits and rlhf.
\newblock \emph{arXiv preprint arXiv:2411.04625}, 2024.

\bibitem[Zhao et~al.(2025{\natexlab{a}})Zhao, Ye, Xiong, Gu, and Zhang]{zhao2025logarithmic}
Heyang Zhao, Chenlu Ye, Wei Xiong, Quanquan Gu, and Tong Zhang.
\newblock Logarithmic regret for online kl-regularized reinforcement learning.
\newblock \emph{arXiv preprint arXiv:2502.07460}, 2025{\natexlab{a}}.

\bibitem[Zhao et~al.(2025{\natexlab{b}})Zhao, Ji, Zhao, Zhang, and Gu]{zhao2025nearly}
Qingyue Zhao, Kaixuan Ji, Heyang Zhao, Tong Zhang, and Quanquan Gu.
\newblock Nearly optimal sample complexity of offline kl-regularized contextual bandits under single-policy concentrability.
\newblock \emph{arXiv preprint arXiv:2502.06051}, 2025{\natexlab{b}}.

\bibitem[Zhao et~al.(2025{\natexlab{c}})Zhao, Ji, Zhao, Zhang, and Gu]{zhao2025sharpanalysisofflinepolicy}
Qingyue Zhao, Kaixuan Ji, Heyang Zhao, Tong Zhang, and Quanquan Gu.
\newblock Towards a sharp analysis of offline policy learning for $f$-divergence-regularized contextual bandits, 2025{\natexlab{c}}.
\newblock URL \url{https://arxiv.org/abs/2502.06051}.

\bibitem[Zhong et~al.(2022)Zhong, Xiong, Zheng, Wang, Wang, Yang, and Zhang]{zhong2022gec}
Han Zhong, Wei Xiong, Sirui Zheng, Liwei Wang, Zhaoran Wang, Zhuoran Yang, and Tong Zhang.
\newblock Gec: A unified framework for interactive decision making in mdp, pomdp, and beyond.
\newblock \emph{arXiv preprint arXiv:2211.01962}, 2022.

\bibitem[Zhou et~al.(2025{\natexlab{a}})Zhou, Wu, and Orabona]{zhou2025unifiedtheoreticalanalysisprivate}
Xingyu Zhou, Yulian Wu, and Francesco Orabona.
\newblock A unified theoretical analysis of private and robust offline alignment: from rlhf to dpo, 2025{\natexlab{a}}.
\newblock URL \url{https://arxiv.org/abs/2505.15694}.

\bibitem[Zhou et~al.(2025{\natexlab{b}})Zhou, Wu, Weng, and Orabona]{zhou2025square}
Xingyu Zhou, Yulian Wu, Wenqian Weng, and Francesco Orabona.
\newblock Square $\chi$po: Differentially private and robust $\chi^2$-preference optimization in offline direct alignment.
\newblock \emph{arXiv preprint arXiv:2505.21395}, 2025{\natexlab{b}}.

\end{thebibliography}

\appendix
\section*{Appendix}
\addcontentsline{toc}{section}{Appendix}
\section{Useful Lemmas}

\begin{lemma}[\citealp{foster2021statistical}]\label{lem:Sequence}
For any sequence of real-valued random variables $\left(X_t\right)_{t \leq T}$ adapted to a filtration $\left(\mathcal{F}_t\right)_{t \leq T}$, it holds that with probability at least $1-\delta$, for all $T^{\prime} \leq T$,
\[
\sum_{t=1}^{T^{\prime}} X_t 
\leq \sum_{t=1}^{T^{\prime}} \log \mathbb{E}_{t-1}\left[e^{X_t}\right] + \log \frac{1}{\delta}~.
\]
\end{lemma}

\begin{lemma}\label{lem:MeanValue}
Let 
\[
f(x) = \log\!\left( \alpha \sigma(x) + (1-\alpha)\big(1 - \sigma(x)\big) \right),
\quad \sigma(x) = \frac{1}{1+e^{-x}},
\]
where $\alpha \in (0.5,1)$ and $x \in [-B,B]$.  
Then for any $a,b \in [-B,B]$, we have
\[
    |f(a) - f(b)| \le \sigma(B)\,|a-b|~.
\]
\end{lemma}
\begin{proof}
First, observe that
\[
\alpha \sigma(x) + (1-\alpha)(1 - \sigma(x))
= 1-\alpha + (2\alpha - 1)\sigma(x)~.
\]
So, we have
\[
f'(x)
= \frac{(2\alpha-1) \sigma(x) (1-\sigma(x))}{1-\alpha + (2\alpha-1) \sigma(x)}
\leq 1- \sigma(x),
\]
where the inequality due to the fact that $1-\alpha\geq 0$.


Maximizing over $x \in [-B, B]$, we obtain 
\[
\sup_{x\in[-B,B],\ \alpha\in(0.5,1)} \ f'(x) 
\leq \sup_{x\in[-B,B]} \ 1-\sigma(x)
= 1-\sigma(-B)
=\sigma(B)~.
\]
Finally, by the Mean Value Theorem, for any $a,b \in [-B,B]$ there exists $c$ between $a$ and $b$ such that
\[
|f(a) - f(b)| 
= |f'(c)|\,|a-b| 
\le \sigma(B)\,|a-b|~. \qedhere
\]
\end{proof}

\begin{lemma}[Freedman's Inequality]\label{lem:Freedman}
    Let $\delta \in (0,1)$.
    Let $M, v>0$ be fixed constants. Let $\left\{X_i\right\}_{i=1}^n$ be a stochastic process, $\left\{\mathcal{G}_i\right\}_i$ be a sequence of $\sigma$-fields, and $X_i$ be $\mathcal{G}_i$-measurable, while almost surely
    \[
    \mathbb{E}\left[X_i \mid \mathcal{G}_i\right]=0,\left|X_i\right| \leq M, \text { and } \sum_{i=1}^n \mathbb{E}\left[X_i^2 \mid \mathcal{G}_{i-1}\right] \leq v~.
    \]
    Then, with probability at least $1-\delta$, it holds that
    \[
    \sum_{i=1}^n X_i \leq \sqrt{2 v \log \frac{1}{\delta}}+\frac{2}{3} M \log \frac{1}{\delta}~.
    \]
\end{lemma}

\begin{lemma}[\cite{zhao2024sharp}]\label{lem:Quadra}
    Suppose $a, b \geq 0$. If $x^2 \leq a+b \cdot x$, then $x^2 \leq  b^2+2 a$.
\end{lemma}

\begin{lemma}[Theorem 1 in \citet{duchi2013local}]\label{lem:LDP_KL}
    For any $\epsilon \geq 0$, let $Q$ be a conditional distribution that guarantees $\epsilon$-local differential privacy. Then for any pair of distributions $P_1$ and $P_2$, the induced marginals $M_1$ and $M_2$ where $M_j(S)=\int_{\mathcal{X}} Q(S \mid x) d P_j(x)$  for  $j=1,2 $ satisfy the bound
    \[
    D_{\mathrm{kl}}\left(M_1 \| M_2\right)+D_{\mathrm{kl}}\left(M_2 \| M_1\right) \leq \min \left\{4, e^{2 \epsilon}\right\}\left(e^\epsilon-1\right)^2\left\|P_1-P_2\right\|_{\mathrm{TV}}^2~.
    \]
\end{lemma}

\begin{lemma}[Assouad's Lemma]\label{lem:assouad}
    Let $\mathcal{I}$ be the set of instances, $\Pi$ be the set of estimators, $\Theta:=\{ \pm 1\}^S$ for some $S>0$, and $\left\{L_j\right\}_{j=1}^S$ be $S$ functions from $\Pi \times \mathcal{I}$ to $\mathbb{R}_{+}$. Suppose $\left\{I_\theta\right\}_{\theta \in \Theta} \subset \mathcal{I}$ and the loss function is
    \[
    L(\pi, I):=\sum_{j=1}^S L_j(\pi, I), \forall(\pi, I) \in \Pi \times \mathcal{I}~.
    \]
    We denote $\theta \sim_j \theta^{\prime}$ if they differ only in the $j$-th coordinate. Further, assume that
    \[
    \theta \sim_j \theta^{\prime} \Rightarrow \inf _{\pi \in \Pi} L_j\left(\pi, J_\theta\right)+L_j\left(\pi, J_{\theta^{\prime}}\right) 
    \geq c,
    \]
    for some $c>0$. Then, we have
    \[
    \inf _{\pi \in \Pi} \ \sup _{I \in \mathcal{I}} \ L(\pi, I) 
    \geq S \cdot \frac{c}{4} \min _{\exists j: \theta \sim_j \theta^{\prime}} \ \exp \left(-\mathrm{KL}\left(P_{I_\theta} \| P_{I_{\theta^{\prime}}}\right)\right),
    \]
    where $P_I$ denotes the distribution of the dataset given $I \in \mathcal{I}$.
\end{lemma}

\begin{lemma}[\citealp{zhao2025sharpanalysisofflinepolicy}]\label{lem:biasFunc}
    Let $b(s): \mathcal{S} \rightarrow \mathbb{R}$ be some bias function, then for all $r(s,a) \in \mathcal{F}$ we have $J\left(\pi_r\right)=J\left(\pi_{r-b}\right)$ since $\pi_r=\pi_{r-b}$ where $\pi_r=\frac{\pi_{ref}\exp{(\beta r)}}{\sum_{a\in\mathcal{A}}\pi_{ref}\exp{(\beta r)}}$ , where $(r-b)(s, a)=r(s, a)-b(s)$.
\end{lemma}

\begin{lemma}\label{lem:PrivFunc}
    Let $\sigma(x)=\frac{1}{1+e^{-x}}$ be  sigmoid function and $f(x)=(2\sigma(x)-1)(2\alpha-1)$ for a fixed $\alpha\in(0.5,1]$.
    For any $B\ge 0$ and any $x,x'\in[-B,B]$,
    \[
    |x-x'|
    \le\frac{e^{-B}+2+e^{B}}{2(2\alpha-1)}\;|f(x)-f(x')|~.
    \]
\end{lemma}
\begin{proof}
    First we have $f'(x)=2(2\alpha-1)\sigma'(x)$ with
    \[
    \sigma'(x)=\frac{e^{-x}}{(1+e^{-x})^{2}}=\frac{1}{e^{x}+2+e^{-x}}~.
    \]
    On $[-B,B]$, $\sigma'$ attains its minimum at $\pm B$:
    \[
    \min_{|x|\le B} \ \sigma'(x)=\frac{1}{e^{B}+2+e^{-B}}~.
    \]
    Hence
    \[
    m:=\inf_{|x|\le B} \ |f'(x)|=\frac{2(2\alpha-1)}{e^{B}+2+e^{-B}}~.
    \]
    By the Mean Value Theorem there exists $\xi$ between $x$ and $x'$ such that
    \[
    |f(x)-f(x')|=|f'(\xi)|\,|x-x'|\ \ge\ m\,|x-x'|,
    \]
    which gives the stated inequality. 
\end{proof}

\section{Proofs of Section \ref{sec:offline}}
\label{sec:AppendixOffline}


In Algorithm~\ref{algo:offline}, we estimate the reward function via MLE. Thus, we extend the approach in \cite{zhao2024sharp} to  establish the generalization error bound of reward difference the MLE, taking into account that here the MLE is on the private probabilities.

\begin{lemma}\label{lem:OffRewardError}
    For an arbitrary policy $\pi$, and a set of offline data $\{(s_i,a_i^1,a_i^2,z_i)\}_{i=1}^n$ generated i.i.d from the BT model and $\pi$, and privatized by RR. Suppose that $\bar{r}$ is the result of the private MLE in step 1 of Algorithm~\ref{algo:offline}, then there exists a function $b(s):\mathcal{S}\rightarrow [-B,B]$ such that with probability at least $1-2\delta$ and for all values of $\tau$ small enough, we have
    \begin{equation}\label{eq:RewardError}
        \mathbb{E}_{s\sim d_0, a\sim \pi(\cdot |s)}[\bar{r}(s,a)-r^*(s,a)-b(s)]^2 =  O\left(\frac{e^B}{(2\alpha-1)^2}\cdot\left(\frac{\log(\mathcal{N}_{\mathcal{F}}(\tau)/\delta)}{n}+ \tau \right)\right)~.
    \end{equation}
\end{lemma}
From the proof of the lemma, define $b(s)=\mathbb{E}_{a\sim \pi(\cdot|s)}[\bar{r}(s,a)-r^*(s,a)]$, then $\mathbb{E}_{s\sim d_0}\operatorname{Var}_{a\sim \pi(\cdot |s)}[\bar{r}(s,a)-r^*(s,a)]=\mathbb{E}_{(s,a)\sim d_0\times\pi}[(\bar{r}(s,a)-r^*(s,a)-b(s))^2]$. Note that, in the offline setting, the actions are sampled from $\pi_{ref}$. 


\begin{proof}[Proof of Lemma \ref{lem:OffRewardError}]
    \textbf{Step 1: Connect private MLE and the reward difference.} Since we estimate the reward function by private MLE, let
    \[
    \widetilde{L}(r|s_i, a_i^1, a_i^2)=\log \left[\alpha \cdot \sigma(z_i \cdot \Delta_r(s_i, a_i^1, a_i^2)) + (1 - \alpha) \cdot \sigma(-z_i \cdot \Delta_r(s_i, a_i^1, a_i^2))\right]~.
    \]
    We first use Lemma~\ref{lem:Sequence} on the sequence 
    \[
    \left\{\frac{1}{2}\widetilde{L}(r|s_i, a_i^1, a_i^2)-\frac{1}{2}\widetilde{L}(r^*|s_i, a_i^1, a_i^2)\right\}_{i=1}^n=\left\{\frac{1}{2}\log \frac{\widetilde{P}_r(z_i|s_i,a_i^1,a_i^2)}{\widetilde{P}_{r^*}(z_i|s_i,a_i^1,a_i^2)}\right\}_{i=1}^n,
    \]
    for any $r \in \mathcal{F}$ where $\widetilde{P}_r$ is defined in \eqref{eq:PrivProb}. Then, for $s \le n$, we have with probability at least $1-\delta$ that
    \begin{align}
         \frac{1}{2}&\sum_{i=1}^s \left[\widetilde{L}(r|s_i, a_i^1, a_i^2)-\widetilde{L}(r^*|s_i, a_i^1, a_i^2)\right] \nonumber\\
         &\le \sum_{i=1}^s \log\mathbb{E} \left[\sqrt{\frac{\widetilde{P}_r(z_i|s_i,a_i^1,a_i^2)}{\widetilde{P}_{r^*}(z_i|s_i,a_i^1,a_i^2)}}\right]+\log \frac{1}{\delta} \nonumber\\
         &= \sum_{i=1}^s \log\left[\sqrt{{\widetilde{P}_r(z_i=-1|s_i,a_i^1,a_i^2)}{\widetilde{P}_{r^*}(z_i=-1|s_i,a_i^1,a_i^2)}}\right.
        \nonumber \\
        & \quad +\left.\sqrt{{\widetilde{P}_r(z_i=+1|s_i,a_i^1,a_i^2)}{\widetilde{P}_{r^*}(z_i=+1|s_i,a_i^1,a_i^2)}}\right]+\log \frac{1}{\delta} \nonumber\\
        &\overset{(a)}{\le}  \sum_{i=1}^s \left[\sqrt{{\widetilde{P}_r(z_i=-1|s_i,a_i^1,a_i^2)}{\widetilde{P}_{r^*}(z_i=-1|s_i,a_i^1,a_i^2)}}\right. \nonumber \\
        & \quad \left.+\sqrt{{\widetilde{P}_r(z_i=+1|s_i,a_i^1,a_i^2)}{\widetilde{P}_{r^*}(z_i=+1|s_i,a_i^1,a_i^2)}}-1\right]+\log \frac{1}{\delta}\nonumber \\
        & = \log \frac{1}{\delta}-\frac12 \sum_{i=1}^s \left(\sqrt{\widetilde{P}_{r^*}(z_i=+1|s_i,a_i^1,a_i^2)}-\sqrt{\widetilde{P}_r(z_i=+1|s_i,a_i^1,a_i^2)}\right)^2 \nonumber\\
        &\quad -\frac12 \sum_{i=1}^s \left(\sqrt{\widetilde{P}_{r^*}(z_i=-1|s_i,a_i^1,a_i^2)}-\sqrt{\widetilde{P}_r(z_i=-1|s_i,a_i^1,a_i^2)}\right)^2 \nonumber\\
        &\overset{(b)}{\le} \log \frac{1}{\delta}-\frac18 \sum_{i=1}^s \left({\widetilde{P}_{r^*}(z_i=+1|s_i,a_i^1,a_i^2)}-{\widetilde{P}_r(z_i=+1|s_i,a_i^1,a_i^2)}\right)^2 \nonumber\\
        &= \log \frac{1}{\delta}-\frac18 \sum_{i=1}^s (2\alpha-1)^2 \cdot [\sigma(\Delta_{r^*}(s_i, a_i^1, a_i^2))-\sigma(\Delta_{r}(s_i,a_i^1, a_i^2))]^2 \nonumber\\
        & \le \log \frac{1}{\delta}-\frac{(2\alpha-1)^2\cdot e^B}{8(1+e^B)^2}\sum_{i=1}^s   [\Delta_{r^*}(s_i, a_i^1, a_i^2)-\Delta_{r}(s_i,a_i^1, a_i^2)]^2, \label{eq:term1}
    \end{align}
 where $(a)$ is from $\log x \le x-1$ for $x>0$, $(b)$ is from $(\sqrt{a}-\sqrt{b})^2 \ge \frac14(a-b)^2 \ \text{for} \ a,b\in[0,1]$ since $(\sqrt{a}-\sqrt{b})^2 = \frac{(a-b)^2}{(\sqrt{a}+\sqrt{b})^2} \ge \frac14(a-b)^2,\ a,b\in[0,1]$ and the last inequality is from $\sigma^\prime(x)\ge \frac{e^B}{(1+e^B)^2}$ for $x\in [-B,B]$.

 \noindent \textbf{Step 2: private likelihood function class well-covered by $\tau$-net of reward function.}  For any $\tau>0$, define $\mathcal{F}_\tau$ as a $\tau$-net for the reward function class $\mathcal{F}$ with covering number $\mathcal{N}_\mathcal{F}(\tau)$ in Definition~\ref{def:Net}. Then, for any $s\in \mathcal{S}, a^1,a^2 \in \mathcal{A}, z \in \{-1,+1\}$ and $r \in \mathcal{F}$, there exists $r^\prime \in \mathcal{F}_\tau$ such that 
 \begin{equation}
     |\widetilde{L}(r|s, a^1, a^2)-\widetilde{L}(r^\prime|s, a^1, a^2)| \le \sigma(B) |\Delta_r(s, a^1, a^2)-\Delta_{r^\prime}(s, a^1, a^2)|
     \le 2\sigma(B)\tau,
 \end{equation} 
where the first inequality is from Lemma \ref{lem:MeanValue} by taking $x=z \cdot \Delta_r(s, a^1, a^2)$ and $\sigma(-x)=1-\sigma(x)$. This yields
\begin{equation}\label{eq:term2}
    \sum_{i=1}^s\widetilde{L}(r|s_i, a_i^1, a_i^2)\le  \sum_{i=1}^s\widetilde{L}(r^\prime|s_i, a_i^1, a_i^2)+2\sigma(B)\tau s~.
\end{equation}

\noindent\textbf{Step 3: confidence bound for the private MLE estimator.} 
Based on \eqref{eq:term1} and the union bound, for all $r^\prime \in  \mathcal{F}_\tau$ we obtain
\[
\frac{1}{2}\sum_{i=1}^n \left[
\widetilde{L}(r^\prime|s_i, a_i^1, a_i^2)-\widetilde{L}(r^*|s_i, a_i^1, a_i^2) \right]
\le \log\frac{\mathcal{N}_{\mathcal{F}}(\tau)}{\delta}-\frac{(2\alpha-1)^2\cdot e^B}{8(1+e^B)^2}\sum_{i=1}^n   [\Delta_{r^*}(s_i, a_i^1, a_i^2)-\Delta_{r^\prime}(s_i,a_i^1, a_i^2)]^2~.
\]
Building on the above inequality and \eqref{eq:term2}, we have with probability at least $1-\delta$, for any $r \in \mathcal{F}$, there exists $r^\prime \in \mathcal{F}_\tau$ such that
 \begin{equation}
 \begin{aligned}
     \frac{1}{2}&\sum_{i=1}^n \left\{\widetilde{L}(r|s_i, a_i^1, a_i^2)-\widetilde{L}(r^*|s_i, a_i^1, a_i^2)\right\}\\
     \le &\frac{1}{2}\sum_{i=1}^n \left\{\widetilde{L}(r|s_i, a_i^1, a_i^2)-\widetilde{L}(r^\prime|s_i, a_i^1, a_i^2)+\widetilde{L}(r^\prime|s_i, a_i^1, a_i^2)-\widetilde{L}(r^*|s_i, a_i^1, a_i^2)\right\}\\
     \overset{(a)}{\le} & \log\frac{\mathcal{N}_{\mathcal{F}}(\tau)}{\delta}-\frac{(2\alpha-1)^2\cdot e^B}{8(1+e^B)^2}\sum_{i=1}^n   [\Delta_{r^*}(s_i, a_i^1, a_i^2)-\Delta_{r^\prime}(s_i,a_i^1, a_i^2)]^2+\sigma(B)\tau n\\
     = & \log\frac{\mathcal{N}_{\mathcal{F}}(\tau)}{\delta}-\frac{(2\alpha-1)^2\cdot e^B}{8(1+e^B)^2}\sum_{i=1}^n   [\Delta_{r^*}(s_i, a_i^1, a_i^2)-\Delta_{r}(s_i,a_i^1, a_i^2)+\Delta_{r}(s_i,a_i^1, a_i^2)-\Delta_{r^\prime}(s_i,a_i^1, a_i^2)]^2+\sigma(B)\tau n\\
    \overset{(b)}{\le}&  \log\frac{\mathcal{N}_{\mathcal{F}}(\tau)}{\delta}-\frac{(2\alpha-1)^2\cdot e^B}{4(1+e^B)^2}\sum_{i=1}^n   [\Delta_{r^*}(s_i, a_i^1, a_i^2)-\Delta_{r}(s_i,a_i^1, a_i^2)]^2+\frac{(2\alpha-1)^2\cdot e^B}{(1+e^B)^2}\tau^2 n+\sigma(B)\tau n\\
    \le &\log\frac{\mathcal{N}_{\mathcal{F}}(\tau)}{\delta}-\frac{(2\alpha-1)^2\cdot e^B}{4(1+e^B)^2}\sum_{i=1}^n   [\Delta_{r^*}(s_i, a_i^1, a_i^2)-\Delta_{r}(s_i,a_i^1, a_i^2)]^2+2 \tau n,
 \end{aligned}
 \end{equation}
 where $(a)$ is from the union bound over $\mathcal{F}_\tau$, $(b)$ is from $(a+b)^2\le 2a^2+2b^2$ and the definition of $\tau$-net for the reward functions, and the last inequality is from the small value of $\tau$. 

Since $\bar{r}$ is the private MLE estimator, by the realizability of the reward function, we have $\sum_{i=1}^n \{\widetilde{L}(\bar{r}|s_i, a_i^1, a_i^2)-\widetilde{L}(r^*|s_i, a_i^1, a_i^2)\} \ge 0$. So, we get
\[
0\le \log\frac{\mathcal{N}_{\mathcal{F}}(\tau)}{\delta}-\frac{(2\alpha-1)^2\cdot e^B}{4(1+e^B)^2}\sum_{i=1}^n   [\Delta_{r^*}(s_i, a_i^1, a_i^2)-\Delta_{\bar{r}}(s_i,a_i^1, a_i^2)]^2+2 \tau n~.
\]
Then, with probability at least $1-\delta$, we have
\begin{equation}
    \sum_{i=1}^n   [\Delta_{r^*}(s_i, a_i^1, a_i^2)-\Delta_{\bar{r}}(s_i,a_i^1, a_i^2)]^2 \le \frac{4(1+e^B)^2}{(2\alpha-1)^2\cdot e^B}\left(\log\frac{\mathcal{N}_{\mathcal{F}}(\tau)}{\delta}+2 \tau n\right)~.
\end{equation}

\noindent\textbf{Step 4: On-policy error bound of reward difference function.} We first get the bound on the finite reward function set $\mathcal{F}_\tau$, then derive it for an infinite set $\mathcal{F}$. 
We now use Lemma~\ref{lem:Freedman} by taking $X_i=\mathbb{E}[Y_i]-Y_i$ as zero mean r.v. where $Y_i=[\Delta_{r^\prime}(s_i, a_i^1, a_i^2)-\Delta_{r^*}(s_i,a_i^1, a_i^2)]^2\in [0,4B^2]$, thus, $|X_i|\le 4B^2$ and $\mathbb{E}X_i^2=\mathbb{E}{[Y_i^2]}-[\mathbb{E}{Y}_i]^2\le \mathbb{E}{Y_i^2} \le 4B^2 \mathbb{E}{Y}_i$.
Hence, by the union bound, with probability at least $1-\delta$ we have for all $r^\prime \in \mathcal{F}_\tau$ that
\begin{equation}
\begin{aligned}
     n& \mathbb{E}_{s\sim d_0,a^1,a^2 \sim \pi} [\Delta_{r^\prime}(s, a^1, a^2)-\Delta_{r^*}(s,a^1, a^2)]^2 -\sum_{i=1}^n   [\Delta_{r^\prime}(s_i, a_i^1, a_i^2)-\Delta_{r^*}(s_i,a_i^1, a_i^2)]^2\\
     &\le \sqrt{4nB^2 \log\frac{\mathcal{N}_{\mathcal{F}}(\tau)}{\delta}\mathbb{E}_{s\sim d_0,a^1,a^2 \sim \pi} [\Delta_{r^\prime}(s, a^1, a^2)-\Delta_{r^*}(s,a^1, a^2)]^2}+\frac{8}{3}B^2 \log\frac{\mathcal{N}_{\mathcal{F}}(\tau)}{\delta} ~.
\end{aligned}
\end{equation}

From the above inequality and by taking $x=\sqrt{n \mathbb{E}_{s\sim d_0,a^1,a^2 \sim \pi} [\Delta_{r^\prime}(s, a^1, a^2)-\Delta_{r^*}(s,a^1, a^2)]^2}, b=2B, a=\frac{8}{3}B^2 \log(\mathcal{N}_\mathcal{F}(\tau)/\delta)+\sum_{i=1}^n   [\Delta_{r^\prime}(s_i, a_i^1, a_i^2)-\Delta_{r^*}(s_i,a_i^1, a_i^2)]^2 $ in Lemma \ref{lem:Quadra}, we get
\[
n \mathbb{E}_{s\sim d_0,a^1,a^2 \sim \pi} [\Delta_{r^\prime}(s, a^1, a^2)-\Delta_{r^*}(s,a^1, a^2)]^2
= O\left(B^2 \log\frac{\mathcal{N}_{\mathcal{F}}(\tau)}{\delta}\right)+\sum_{i=1}^n   [\Delta_{r^\prime}(s_i, a_i^1, a_i^2)-\Delta_{r^*}(s_i,a_i^1, a_i^2)]^2~.
\]

By the definition of $\tau$-net in Definition \ref{def:Net}, we have for the private MLE estimator $\bar{r}$, there exists a $r^\prime \in \mathcal{F}_\tau$, such that, for all $(s,a)\in \mathcal{S}\times\mathcal{A}$, we have $|r^\prime(s,a)-\bar{r}(s,a)|\le \tau$ from which and the result in step 3 we can further derive with probability at least $1-2\delta$

\[\begin{aligned}
    &\mathbb{E}_{s\sim d_0,a^1,a^2 \sim \pi} [\Delta_{\bar{r}}(s, a^1, a^2)-\Delta_{r^*}(s,a^1, a^2)]^2\\
    &\quad = O\left(\frac{B^2 \log(\mathcal{N}_\mathcal{F}(\tau)/\delta)}{n}\right)+\frac1n\sum_{i=1}^n   [\Delta_{r^\prime}(s_i, a_i^1, a_i^2)-\Delta_{r^*}(s_i,a_i^1, a_i^2)]^2+8\tau^2\\
    &\quad = \frac{4(1+e^B)^2}{(2\alpha-1)^2\cdot e^B}\cdot\left(\frac{\log(\mathcal{N}_{\mathcal{F}}(\tau)/\delta)}{n}+2 \tau \right)+O\left(\frac{B^2 \log(\mathcal{N}_\mathcal{F}(\tau)/\delta)}{n}\right)+8\tau^2\\
    &\quad = O\left(\frac{e^B}{(2\alpha-1)^2}\cdot\left(\frac{\log(\mathcal{N}_{\mathcal{F}}(\tau)/\delta)}{n}+ \tau \right)\right),
\end{aligned}
\]
for all values of $\tau$ small enough. Then, we get the result by taking $b(s)=\mathbb{E}_{a^2\sim \pi}[\bar{r}(s,a^2)-r^*(s,a^2)]$.
\end{proof}


\begin{lemma}\label{lem:UpperSub}

From Lemma 2.16 in \cite{zhao2025sharpanalysisofflinepolicy} and Lemma E.2 in \cite{zhao2025sharpanalysisofflinepolicy}, if pessimistic event $(g-r^* )(s,a)\le 0$ holds, we have
\[
J\left(\pi^*\right)-J\left(\pi_g\right) 
\leq \beta \mathbb{E}_{(s, a) \sim \rho \times \pi^*}\left[\left(g-r^*\right)^2(s, a)\right]~.
\]
\end{lemma}

We state the details of the proof here.
\begin{proof}[Proof of Theorem \ref{thm:offlineUpper}]
    Similar to Lemma E.1 in \cite{zhao2025sharpanalysisofflinepolicy}, it is easy to get with probability at least $1-\delta$, the event $\mathcal{E}(\delta):=\left\{\exists b: \mathcal{S} \rightarrow[-B,B], \forall(s, a) \in \mathcal{S} \times \mathcal{A},\left|\bar{r}(s, a)-b(s)-r^*(s, a)\right| \leq \Gamma_n(s, a)\right\}$ holds for $\delta \in (0,1)$.

From the result of Lemma \ref{lem:OffRewardError}, we have with probability at least $1-\delta$,
\[\mathbb{E}_{s^\prime\sim d_0}\text{Var}_{a^\prime \sim \pi_{\text{ref}}}[\bar{r}(s^\prime,a^\prime)-r^*(s^\prime,a^\prime)]\le O\left(\frac{e^B}{(2\alpha-1)^2}\cdot\left(\frac{\log(\mathcal{N}_{\mathcal{F}}(\tau)/\delta)}{n}+ \tau \right)\right).\] Then we have 
\[\begin{aligned}
     \inf_{b} \ &(\bar{r}(s,a)-b(s)-r^*(s,a))^2\\
    &= \inf_{b}\frac{(\bar{r}(s,a)-b(s)-r^*(s,a))^2}{\mathbb{E}_{s^\prime\sim d_0}\text{Var}_{a^\prime \sim \pi_{\text{ref}}}[\bar{r}(s^\prime,a^\prime)-r^*(s^\prime,a^\prime)]} \mathbb{E}_{s^\prime\sim d_0}\text{Var}_{a^\prime \sim \pi_{\text{ref}}}[\bar{r}(s^\prime,a^\prime)-r^*(s^\prime,a^\prime)]\\
    &\le D_{\mathcal{F}}^2((s,a),\pi_{\text{ref}})\mathbb{E}_{s^\prime\sim d_0}\text{Var}_{a^\prime \sim \pi_{\text{ref}}}[\bar{r}(s^\prime,a^\prime)-r^*(s^\prime,a^\prime)]\\
    &\le D_{\mathcal{F}}^2((s,a),\pi_{\text{ref}})  O\left(\frac{e^B}{(2\alpha-1)^2}\cdot\left(\frac{\log(\mathcal{N}_{\mathcal{F}}(\tau)/\delta)}{n}+ \tau \right)\right)~.\end{aligned}
    \]

Thus, we get $\mathcal{E}(\delta)$ holds with probability at least $1-\delta$.

Under event $\mathcal{E}(\delta)$, we have  $\hat{r}(s,a)-b(s)\le r^*(s,a)$,
\[
J(\pi^*)-J(\pi_{\hat{r}})=J(\pi^*)-J(\pi_{\hat{r}-b})\le \beta  \cdot \mathbb{E}_{(s,a)\sim d_0\times \pi^*}[(\hat{r}(s,a)-b(s)-r^*(s,a))^2],
\]
where $\hat{r}(s,a)=\bar{r}(s,a)-\Gamma_n(s,a)$ in Step 2 of Algorithm \ref{algo:offline}, the equation is from Lemma \ref{lem:biasFunc} and  the inequality is from Lemma \ref{lem:UpperSub}. Therefore, we obtain
    \[
    \begin{aligned}
        J(\pi^*)-J(\pi_{\hat{r}}) &\le \beta  \cdot \mathbb{E}_{(s,a)\sim d_0\times \pi^*}[(\hat{r}(s,a)-b(s)-r^*(s,a))^2]\\
        &= \beta \cdot \mathbb{E}_{(s,a)\sim d_0\times \pi^*}[(\bar{r}(s,a)-\Gamma_n(s,a)-b(s)-r^*(s,a))^2]\\
        &\le  \beta \left(2\mathbb{E}_{(s,a)\sim d_0\times \pi^*}  [\Gamma_n(s,a)]^2+2\mathbb{E}_{(s,a)\sim d_0\times \pi^*}[(\bar{r}(s,a)-b(s)-r^*(s,a))^2]\right)\\
        & \le 4 \beta \mathbb{E}_{(s,a)\sim d_0\times \pi^*}  [\Gamma_n(s,a)]^2\\
        & = 4\beta  D_{\pi^*}^2 \cdot  O\left(\frac{e^B}{(2\alpha-1)^2}\cdot\left(\frac{\log(\mathcal{N}_{\mathcal{F}}(\tau)/\delta)}{n}+ \tau \right)\right) \\
        & = O\left(\beta  D_{\pi^*}^2 \frac{e^B}{(2\alpha-1)^2}\cdot\left(\frac{\log(\mathcal{N}_{\mathcal{F}}(\tau)/\delta)}{n}+ \tau \right) \right),
    \end{aligned}
    \]
\end{proof}





\begin{proof}[Proof of Theorem \ref{thm:offlineLower}]
Consider the set of private RLHF instances
\[
\mathcal{I}
=\{(\mathcal{S},\mathcal{A},r,\pi_{ref},\beta,\mathcal{R})\},
\]
where $\mathcal{R}$ is the LDP randomizer. We aim to construct a specific instance in the set to get the minimax lower bound.

\noindent\textbf{Step 1: Construct the instance.} Inspired by \cite{zhao2025sharpanalysisofflinepolicy}, we consider the following instance for the private RLHF problem via the contextual dueling bandits view: the state space $\mathcal{S}=[S]$ where $S\ge1$, binary action space  $\mathcal{A}=\{-1,+1\}$, $d_0=Unif{(\mathcal{S})}$ is a uniform distribution, the reward function in some function class $\mathcal{F}\subseteq \mathcal{S}\times\mathcal{A}\rightarrow [0,B]$ and the reference policy for any $s\in \mathcal{S}$ to be 
    \[
    \pi_{ref}(-1|s)=1/C, \quad \pi_{ref}(+1|s)=1-1/C,
    \]
    where $C\ge 1$ is a parameter to be decided later. We consider collections of distributions indexed using the Boolean hypercube $\mathcal{V}=\{-1,+1\}^S$. In particular, for any $\mathbf{v}=(v_1,v_2,\dots,v_S)\in \mathcal{V}$, the mean function of the reward indexed by $\mathbf{v}$ is defined as 
    \[
    r_{\mathbf{v}}(s,-1)=B/2+v_s\cdot a, \quad r_{\mathbf{v}}(s,+1)=B/2-b,
    \]
    for any state $s\in \mathcal{S}$, where $a,b\in (0,B/2)$ will be specified later. With this reward function, from Definition \ref{eq:PolicyImprove}, the optimal policy $\pi^*_{\mathbf{v}}$ for the KL-regularized RLHF is for any $s\in \mathcal{S}$,
    \begin{align}
    \pi^*_{\mathbf{v}}(-1|s)
    &=\frac{\pi_{ref}(-1|s)\exp{(\beta\cdot r_{\mathbf{v}}(s,-1))}}{\pi_{ref}(-1|s)\exp{(\beta\cdot r_{\mathbf{v}}(s,-1))}+\pi_{ref}(+1|s)\exp{(\beta\cdot r_{\mathbf{v}}(s,+1))}}=\frac{\exp(\beta(b+v_s a))}{\exp(\beta(b+v_s a))+C-1},\nonumber\\
    \pi^*_{\mathbf{v}}(+1|s)
    &=\frac{C-1}{\exp(\beta(b+v_s a))+C-1}~. \nonumber
    \end{align}
    
\noindent\textbf{ Step 2: Verify the single policy concentrability.}  Following \cite{zhao2025sharpanalysisofflinepolicy}, we state the verification for concentrability here for completeness. Set $C^*\ge 2$, $C=C^*$ and $b=\beta^{-1}\log(C-1)$, then for any $s\in \mathcal{S}$,
\begin{align*}
\frac{\pi^*_{\mathbf{v}}(-1|s)}{\pi_{ref}(-1|s)}
&=C\cdot \frac{\exp(\beta(b+v_s a))}{\exp(\beta(b+v_s a))+C-1}=C\cdot \frac{\exp(\beta v_s a)}{1+\exp(\beta v_s a)}\le C=C^*,\\
\frac{\pi^*_{\mathbf{v}}(+1|s)}{\pi_{ref}(+1|s)}
&=\frac{C}{C-1}\cdot\frac{1}{1+\exp(\beta v_s a)}\le C=C^*~.
\end{align*}
Therefore, we get $\max_{\mathbf{v}\in \mathcal{V}} C^{\pi^*_{\mathbf{v}}} \le C^*$.

\noindent\textbf{Step 3: Construction of hard-to-distinguish pair for Sub-optimality gap.} In order to get the minimax lower bound, since $d_0=Unif(\mathcal{S})$, we define 
\[  \text{SubOpt}(\hat{\pi},\mathbf{v})=\frac{1}{S}\sum_{s=1}^S\text{SubOpt}_s(\hat{\pi},\mathbf{v}),\]
and, simpler than the analysis in \cite{zhao2025sharpanalysisofflinepolicy}, we have the following derivation from sub-optimality gap to the KL divergence between estimated policy and optimal policy:
\[
\begin{aligned}
    \text{SubOpt}_s(\hat{\pi},\mathbf{v})&=\left\langle\pi_\mathbf{v}^*(\cdot \mid s), r_\mathbf{v}(s, \cdot)-\beta^{-1} \log \frac{\pi_\mathbf{v}^*(\cdot \mid s)}{\pi_{\mathrm{ref}}(\cdot \mid s)}\right\rangle-\left\langle\widehat{\pi}(\cdot \mid s), r_\mathbf{v}(s, \cdot)-\beta^{-1} \log \frac{\widehat{\pi}(\cdot \mid s)}{\pi_{\mathrm{ref}}(\cdot \mid s)}\right\rangle \\
    &=\frac{1}{\beta}\mathbb{E}_{a\sim \pi^*_\mathbf{v}(\cdot|s)}\left[\log\frac{\pi_{\mathrm{ref}}(a|s)\cdot \exp(\beta r_\mathbf{v}(s, a)) }{\pi_\mathbf{v}^*(a \mid s)}\right]-\frac{1}{\beta}\mathbb{E}_{a\sim \hat{\pi}(\cdot|s)}\left[\log\frac{\pi_{\mathrm{ref}}(a|s)\cdot \exp(\beta r_\mathbf{v}(s, a)) }{\hat{\pi}(a \mid s)}\right]\\
    &\overset{(a)}{=} \frac{1}{\beta} \log Z(s)-\frac{1}{\beta}\mathbb{E}_{a\sim \hat{\pi}(\cdot|s)}\left[\log\frac{\pi_{\mathrm{ref}}(a|s)\cdot \exp(\beta r_\mathbf{v}(s, a)) }{\pi_\mathbf{v}^*(a \mid s)}\cdot \frac{\pi_\mathbf{v}^*(a \mid s)}{\hat{\pi}(a \mid s)}\right]\\
    &\overset{(b)}{=} \frac{1}{\beta} \log Z(s) - \frac{1}{\beta} \log Z(s)+\frac{1}{\beta}\mathbb{E}_{a\sim \hat{\pi}(\cdot|s)}\left[\log\frac{\hat{\pi}(a \mid s)}{\pi_\mathbf{v}^*(a \mid s)}\right]\\
    &=\frac{1}{\beta} \mathrm{KL}(\hat{\pi}\|\pi_\mathbf{v}^*),
\end{aligned}
\]
where $(a),(b)$ is from the definition of $ \pi^*_\mathbf{v}(\cdot|s)=\frac{\pi_{\mathrm{ref}}(\cdot|s)\cdot \exp(\beta r_\mathbf{v}(s, \cdot))}{Z(s)}$ and $\mathbb{E}_{a\sim \pi^*_\mathbf{v}(\cdot|s)}Z(s)=Z(s)=\mathbb{E}_{a\sim \hat{\pi}(\cdot|s)}Z(s)$ is the normalization constant.

We denote $\mathbf{v}\sim_s{\mathbf{v^\prime}}$ if  $\mathbf{v},\mathbf{v^\prime} \in \mathcal{V}=\{-1,+1\}^S$ only differ in the $s$-th element and $\mathbf{v}\sim {\mathbf{v^\prime}}$ means there exists $s\in \mathcal{S}, \mathbf{v}\sim_s{\mathbf{v^\prime}}$. By following the equations of (B.10) and (B.11) in Appendix~B.4 of \cite{zhao2025sharpanalysisofflinepolicy} and taking $C-1=\exp(\beta b)$, for any $s\in \mathcal{S}$, we consider $\mathbf{v}\sim_s{\mathbf{v^\prime}}$ and obtain
\[ \text{SubOpt}_s(\hat{\pi},\mathbf{v})+ \text{SubOpt}_s(\hat{\pi},\mathbf{v}^\prime)\ge \min\left\{\frac{\beta a^2}{8},\frac{3a}{10}\right\}~.
\]

\noindent\textbf{Step 4: LDP mechanism on labels.} Let $P_r$ be the distribution of $\left(s, a^1, a^2, z\right)$ for $s \sim d_0, a^1=-1, a^2=+1 \stackrel{\text { i.i.d. }}{\sim} \pi_{\text {ref }}(\cdot \mid s)$, $z=\mathcal{R}(y)$ with LDP randomizer $\mathcal{R}$  and $y \sim \operatorname{Bern}\left(\sigma\left(r\left(s, a^1\right)-r\left(s, a^2\right)\right)\right)$. Note that for the value of the KL divergence the $\{-1,+1\}$ labels are the same as $\{0,1\}$ labels. Then for $\mathbf{v} \sim \mathbf{v}^{\prime}$ with $v_s=-v_s^{\prime}$,
\[
\begin{aligned}
\mathrm{KL}&\left(P_{r_\mathbf{v}} \| P_{r_{\mathbf{v}^{\prime}}}\right) \\
&\le  \frac{(C-1)}{S C^2} \sum_{s^{\prime}, a^1, a^2} [\mathrm{KL}(\mathcal{R}(y_\mathbf{v})\|\mathcal{R}(y_{\mathbf{v}^\prime}))+\mathrm{KL}(\mathcal{R}(y_{\mathbf{v}^\prime})\|\mathcal{R}(y_\mathbf{v}))]\\
&\le  \frac{4(e^\epsilon-1)^2(C-1)}{S C^2} \sum_{s^{\prime}, a^1, a^2} \mathrm{TV}^2\left(\operatorname{Bern}\left(\sigma\left(r_\mathbf{v}\left(s^{\prime}, a^1\right)-r_\mathbf{v}\left(s^{\prime}, a^2\right)\right)\right) \| \operatorname{Bern}\left(\sigma\left(r_{\mathbf{v}^{\prime}}\left(s^{\prime}, a^1\right)-r_{\mathbf{v}^\prime}\left(s^{\prime}, a^2\right)\right)\right)\right) \\
&= \frac{4(e^\epsilon-1)^2(C-1)}{S C^2}\mathrm{TV}^2(\operatorname{Bern}(\sigma(b+a)) \| \operatorname{Bern}(\sigma(b-a))) \\
&= \frac{4(e^\epsilon-1)^2(C-1)}{S C^2} \left(\frac{1}{1+e^{-(a+b)}}-\frac{1}{1+e^{a-b}}\right)^2\\
&\overset{(a)}{\le} \frac{(e^\epsilon-1)^2 a^2}{S C},
\end{aligned}
\]
where the second inequality is from Lemma \ref{lem:LDP_KL} since the offline setting is non-interactive and $(a)$ is from mean-value theorem \[
\big|\sigma(b+a)-\sigma(b-a)\big|
\le\sup_{t\in[b-a,b+a]} \ \big|\sigma'(t)\big| \cdot \big|(b+a)-(b-a)\big|
\le\frac{1}{4}\cdot 2|a|
=\frac{|a|}{2}~.
\]

\noindent\textbf{Step 5: Minimax lower bound.} We evaluate procedures through the minimax suboptimality, which means among all algorithms, pick the one that achieves the smallest possible worst-case suboptimality. From Assouad's lemma in Lemma \ref{lem:assouad} and by taking $a=\frac{\sqrt{SC}}{(e^\epsilon-1)\sqrt{n}}$, $S=\log\mathcal{N}_\mathcal{F}(\tau)$, and $C=C^*$, we get
\[
\begin{aligned}
    \inf_{\hat{\pi}\in \Pi} \ \sup_{I \in \mathcal{I}} \ \text{SubOpt}(\hat{\pi},I)
    &\ge \frac{1}{4} S\cdot \frac{1}{S} \min\left\{\frac{\beta a^2}{8},\frac{3a}{10}\right\} \min_{\mathbf{v}\sim \mathbf{v}^\prime}\exp{\left(-\mathrm{KL}\left(P^n_{r_\mathbf{v}} \| P^n_{r_{\mathbf{v}^{\prime}}}\right)\right)}\\
    &= \frac{1}{4}  \min\left\{\frac{\beta a^2}{8},\frac{3a}{10}\right\} \exp{\left(-n\mathrm{KL}\left(P_{r_\mathbf{v}} \| P_{r_{\mathbf{v}^{\prime}}}\right)\right)}\\
    &= \Omega\left(\min\left\{\frac{\beta CS}{(e^\epsilon-1)^2 n},\frac{\sqrt{SC}}{(e^\epsilon-1)\sqrt{n}}\right\}\right)\\
    &= \Omega\left(\min\left\{\frac{\beta C^* \log\mathcal{N}_\mathcal{F}(\tau)}{(e^\epsilon-1)^2 n},\frac{\sqrt{C^*\log\mathcal{N}_\mathcal{F}(\tau)}}{(e^\epsilon-1)\sqrt{n}}\right\}\right)~. \qedhere
\end{aligned}
\]
\end{proof}

    

\section{Proofs of Section \ref{sec:online}}

By direct calculation, it is easy to get the following lemma that will be used in our follow-up analysis.
\begin{lemma}\label{lem:MeanFunc}
    From the Bernoulli distribution of $y$ in Bradley-Terry model (Assumption \ref{eq:BTmodel}), we denote $\mathbb{E}_r[y|s,a^1,a^2]=h^*(s,a^1,a^2)=2\sigma(\Delta_{r^*}(s,a^1,a^2))-1$, then based on the randomness of random response, $\mathbb{E}_{RR}[z|s,a^1,a^2]=\widetilde{h}^*(s,a^1,a^2)={(2\alpha-1)\cdot h^*(s,a^1,a^2)}$.
\end{lemma}
\begin{proof}[Proof of Lemma \ref{lem:MeanFunc}] First,
    $\mathbb{E}_r[y|s,a^1,a^2]=(+1)\cdot \sigma(\Delta_{r^*}(s,a^1,a^2))+(-1)\cdot (1- \sigma(\Delta_{r^*}(s,a^1,a^2)))=2 \sigma(\Delta_{r^*}(s,a^1,a^2))-1 =h^*(s,a^1,a^2).$
    Then, $\mathbb{E}_{RR}[z|s,a^1,a^2]=1\cdot \mathbb{P}(z=+1|s,a^1,a^2)+(-1)\cdot\mathbb{P}(z=-1|s,a^1,a^2)= \alpha \mathbb{P}(y=+1|s,a^1,a^2)+(1-\alpha) \mathbb{P}(y=-1|s,a^1,a^2)-\alpha \mathbb{P}(y=-1|s,a^1,a^2)-(1-\alpha)\mathbb{P}(y=+1|s,a^1,a^2)=(2\alpha-1)h^*(s,a^1,a^2).$
\end{proof}

\begin{lemma}[In-sample error of ERM \citep{zhao2025logarithmic,zhang2023math,ye2023corruption}]\label{lem:InSample}
     Consider a function space $\mathcal{H}: \mathcal{Z} \rightarrow \mathbb{R}$ and a filtered sequence $\left\{x_t, \epsilon_t\right\} \in \mathcal{X} \times \mathbb{R}$ so that $\epsilon_t$ is conditional zero-mean $\sigma$-sub-Gaussian noise. Suppose that $\mathcal{H}$ is a finite space with cardinality $N_{\mathcal{H}}$. For $h^*(\cdot): \mathcal{Z} \rightarrow \mathbb{R}$, suppose that $z_t=h^*\left(x_t\right)+\epsilon_t$. If $\widehat{f}_t$ is an ERM solution:
$$
\hat{h}_t=\underset{h \in \mathcal{H}}{\operatorname{argmin}} \sum_{i=1}^t\left(h\left(x_i\right)-z_i\right)^2,
$$
with probability at least $1-\delta$, we have for all $t \in[T]$,
\[
\sum_{i=1}^t\left(\widehat{h}_t\left(x_i\right)-h^*\left(x_i\right)\right)^2 \leq 8 \sigma^2 \log \frac{T \cdot N_{\mathcal{F}}}{\delta}~.
\]
\end{lemma}


\begin{lemma}[In sample error bound of reward difference]\label{lem:InSampleRewardDiff}
    Under Assumption \ref{Assum:BT}, finite reward space $\mathcal{F}$ with cardinality $N_{\mathcal{F}}$, the reward $\bar{r}$ estimated by step 7 in Algorithm \ref{algo:online} satisfies w ith probability at least $1-\delta$, for all $t\in [T]$,
    \[
    \sum_{i=1}^{t}\left(r^*\left(s_i, a_i^1\right)-r^*\left(s_i, a_i^2\right)-[\bar{r}_t\left(s_i, a_i^1\right)-\bar{r}_t\left(s_i, a_i^2\right)]\right)^2\le \frac{8(e^{-B}+2+e^{B})^2}{(2\alpha-1)^2} \log \frac{T \cdot N_{\mathcal{F}}}{\delta}~.
    \]
\end{lemma}
\begin{proof}

    By the mean value theorem 
    from Lemma \ref{lem:InSample} and Lemma \ref{lem:PrivFunc} where the noise is from random response with zero-mean $2$-sub-Gaussian noise based on Lemma \ref{lem:MeanFunc}, with probability at least $1-\delta$, we have for all $t \in[T]$
\[\begin{aligned}
    \sum_{i=1}^{t}\left(r^*\left(s_i, a_i^1\right)-r^*\left(s_i, a_i^2\right)-[\bar{r}_t\left(s_i, a_i^1\right)-\bar{r}_t\left(s_i, a_i^2\right)]\right)^2
    &\le \frac{(e^{-B}+2+e^{B})^2}{4(2\alpha-1)^2}\sum_{i}(\hat{\widetilde{h}}_t-\widetilde{h}^*)^2\\
    &\le \frac{8(e^{-B}+2+e^{B})^2}{(2\alpha-1)^2} \log \frac{T \cdot N_{\mathcal{H}}}{\delta}\\
    &\le \frac{8(e^{-B}+2+e^{B})^2}{(2\alpha-1)^2} \log \frac{T \cdot N_{\mathcal{F}}}{\delta}\\
    &=\frac{1}{2} \Gamma_T^2,
\end{aligned}\]
where the last inequality is since  $N_{\mathcal{H}}\le N_{\mathcal{F}}$.
\end{proof}

\begin{lemma}\label{lem:optUpp}
    Under Algorithm \ref{algo:online} and Assumption \ref{Assum:BT}, the noises of the random response on labels $\{-1,+1\}$ are zero mean $2$-sub-Gaussian, we have with probability $1-\delta$, the optimism event that $\mathcal{E}_t=\{\bar{r}_t(s,a)+b_t(s,a)+c_t(s)-r^*(s,a)\ge 0\}$ holds for any $(s,a)\in \mathcal{S}\times\mathcal{A}$ for all $t\in [T]$ uniformly where $c_{t}(s)=\mathbb{E}_{b\sim \pi_{t+1}^1}[{r}^*(s, b)-\bar{r}_{t}(s, b)]$.
\end{lemma}
\begin{proof}
    For any $(s,a)\in \mathcal{S}\times \mathcal{A}$, we have 
    \[
    \begin{aligned}
        &|r^*(s,a)-\bar{r}_t(s,a)-c_t(s)|\\
        &\quad \le \frac{|r^*(s,a)-\bar{r}_t(s,a)-c_t(s)|}{\sqrt{\lambda+\sum_{i=1}^{t}\left(r^*\left(s_i, a_i^1\right)-r^*\left(s_i, a_i^2\right)-[\bar{r}_t\left(s_i, a_i^1\right)-\bar{r}_t\left(s_i, a_i^2\right)]\right)^2}}\\
        &\qquad \cdot \sqrt{\lambda+\sum_{i=1}^{t}\left(r^*\left(s_i, a_i^1\right)-r^*\left(s_i, a_i^2\right)-[\bar{r}_t\left(s_i, a_i^1\right)-\bar{r}_t\left(s_i, a_i^2\right)]\right)^2}\\
        &\quad \le  \sup _{r_1, r_2 \in \mathcal{F}_t} \frac{\left|r_1(s, a)-r_2(s, a)-\mathbb{E}_{b\sim \pi_{t+1}^1}[r_1(s, b)-r_2(s, b)]\right|}{\sqrt{\lambda+\sum_{i=1}^{t}\left(r_1\left(s_i, a_i^1\right)-r_1\left(s_i, a_i^2\right)-[r_2\left(s_i, a_i^1\right)-r_2\left(s_i, a_i^2\right)]\right)^2}} \\
        &\qquad \cdot \sqrt{\lambda+\sum_{i=1}^{t}\left(r^*\left(s_i, a_i^1\right)-r^*\left(s_i, a_i^2\right)-[\bar{r}_t\left(s_i, a_i^1\right)-\bar{r}_t\left(s_i, a_i^2\right)]\right)^2}\\
        &\quad = U_{\mathcal{F}_t}\left(\lambda, s, a ; \mathcal{D}_{t};\pi_{t+1}^1\right) \cdot \sqrt{\lambda+\sum_{i=1}^{t}\left(r^*\left(s_i, a_i^1\right)-r^*\left(s_i, a_i^2\right)-[\bar{r}_t\left(s_i, a_i^1\right)-\bar{r}_t\left(s_i, a_i^2\right)]\right)^2}\\
        &\quad \le U_{\mathcal{F}_t}\left(\lambda, s, a ; \mathcal{D}_{t};\pi_{t+1}^1\right) \cdot\sqrt{\lambda+\frac{1}{2} \Gamma_T^2}\\
        &\quad \le U_{\mathcal{F}_t}\left(\lambda, s, a ; \mathcal{D}_{t};\pi_{t+1}^1\right) \cdot \Gamma_T\\
        &\quad = b_t(s,a),
    \end{aligned}
    \]
    where the last inequality is from taking $\lambda \le \frac{1}{2} \Gamma_T^2$.
\end{proof}


\begin{lemma}[Objective Decomposition, Lemma A.1 in \cite{zhao2025logarithmic}]\label{lem:ObjDec}
    For any $t \in[T]$, conditioning on the uniform optimism event that $\mathcal{E}_t= \left\{\bar{r}_t(x, a)+b_t(x, a)-r^*(x, a) \geq 0, \forall(x, a) \in \mathcal{X} \times \mathcal{A}\right\}$ holds, we have

$$
J\left(\pi^*\right)-J\left(\pi_t\right) \leq \beta \mathbb{E}_{x \sim d_0} \mathbb{E}_{a \sim \pi_t}\left[\left(\bar{r}_{t-1}(s, a)+b_{t-1}(s, a)-r^*(s, a)\right)^2\right] .
$$
where $\pi_t=\pi_{(\bar{r}_{t-1}+b_{t-1})(s, a)}$.
\end{lemma}

\begin{proof}[Proof of Theorem \ref{thm:onlineRegret}]
Based on the uniform event that $\cup_{t\in [T]} \mathcal{E}_t$ holds with probability at least $1-\delta$, 
and denoting $c_{t-1}(s)=\mathbb{E}_{b\sim \pi_t^1}[{r}^*(s, b)-\bar{r}_{t-1}(s, b)]$, 
from Lemma \ref{lem:biasFunc}, we have 
\[
J(\pi^*)-J(\pi^2_t)=J(\pi^*) - J(\pi_{\bar{r}_{t-1}+b_{t-1}})=J(\pi^*) - J(\pi_{(\bar{r}_{t-1}+b_{t-1})(s,a)+c_{t-1}(s)})~.
\]

From Lemma \ref{lem:ObjDec} for objective decomposition, under the event $\mathcal{E}_t$, we have 
\[J(\pi^*)-J(\pi_t^2)\le \beta \mathbb{E}_{s\sim d_0}\mathbb{E}_{a\sim \pi_t^2}[(\bar{r}_{t-1}(s,a)+b_{t-1}(s,a)+c_{t-1}(s)-r^*(s,a))^2] \le 4\beta \mathbb{E}_{s\sim d_0}\mathbb{E}_{a\sim \pi_t^2} [b_{t-1}(s,a)]^2~.~.\]
where the last inequality is from Lemma \ref{lem:optUpp}.

Thus, we get the cumulative regret bound is 
\[\begin{aligned}
    \sum_{t=1}^T (J(\pi^*) -J(\pi^2_t)) \leq \sum_{t=1}^T 4\beta \mathbb{E}_{s\sim d_0}\mathbb{E}_{a\sim \pi_t^2} [b_{t-1}(s,a)]^2~.
\end{aligned}\]
By plugging in $b_t(s,a)= U_{\mathcal{F}_t}\left(\lambda, s, a ; \mathcal{D}_{t};\pi_{t+1}^1\right) \cdot \Gamma_T$, we get the final result.
\end{proof}

\end{document}